\documentclass[runningheads]{llncs}

\usepackage[
  paper=a4paper,
  margin=1.2in 
]{geometry}

\makeatletter
\@twosidefalse
\@mparswitchfalse
\makeatother

\usepackage{eccv}

\usepackage{caption}

\usepackage{colortbl}

\definecolor{geoPrimary}{RGB}{0, 113, 178}    
\definecolor{geoDark}{RGB}{0, 60, 60}           
\definecolor{geoLight}{RGB}{0, 158, 116}    
\definecolor{geoAccent}{RGB}{34, 139, 34}      
\definecolor{linearColor}{RGB}{70, 110, 180}  
\definecolor{paperWhite}{RGB}{250, 250, 250}  
\definecolor{geoPurple}{RGB}{112, 48, 160}
\usepackage{fontawesome5}
\usepackage[framemethod=TikZ]{mdframed}
\newcounter{takeawaycounter}
\newenvironment{takeaway}{
  \stepcounter{takeawaycounter}
  \renewcommand{\thetakeawaycounter}{\Roman{takeawaycounter}}
  \begin{mdframed}[
    middlelinecolor   = geoLight, 
    middlelinewidth   = 2pt,      
    backgroundcolor   = geoLight!3, 
    roundcorner       = 4pt,
    innertopmargin   = 6pt,      
    innerbottommargin = 6pt,
    skipabove=10pt,            
    skipbelow=10pt,            
  ]
  \textbf{\textcolor{geoLight}{\faLightbulb\hspace{8pt}Takeaway~\thetakeawaycounter:}} 
}{
  \end{mdframed}
}

\usepackage{xcolor}
\usepackage{colortbl}
\usepackage{textgreek}
\usepackage{multirow}
\usepackage{booktabs}

\definecolor{colorFedAvg}{HTML}{F0E442}
\definecolor{colorFedDyn}{HTML}{E69F00}
\definecolor{colorFedProx}{HTML}{D55E00}
\definecolor{colorScaffold}{HTML}{882255}
\definecolor{colorFederatedFactoryCentralized}{HTML}{56B4E9}
\definecolor{colorFederatedFactoryDecentralized}{HTML}{0072B2}
\definecolor{colorUpperBound}{HTML}{009E73}

\newenvironment{hypothesis}{
  \begin{mdframed}[
    middlelinecolor   = geoPrimary, 
    middlelinewidth   = 2pt,      
    backgroundcolor   = geoPrimary!3, 
    roundcorner       = 4pt,
    innertopmargin   = 6pt,      
    innerbottommargin = 6pt,
    skipabove=10pt,            
    skipbelow=10pt,            
  ]
  \textbf{\textcolor{geoPrimary}{\faBook\hspace{8pt}Hypothesis}:}
}{
  \end{mdframed}
}

\usepackage{eccvabbrv}

\usepackage{graphicx}
\usepackage{booktabs}

\usepackage[accsupp]{axessibility}  

\usepackage{xcolor}
\usepackage{graphicx}
\usepackage{subcaption}
\usepackage{caption}      
\usepackage{tcolorbox}

\usepackage{booktabs}     
\usepackage{multirow}     

\usepackage{booktabs}
\usepackage{tabularx}
\usepackage{ragged2e} 

\usepackage[labelfont=bf, font=small]{caption}

\usepackage{hyperref}

\usepackage{orcidlink}

\usepackage[utf8]{inputenc}
\usepackage[T1]{fontenc}
\usepackage{graphicx}
\usepackage{xcolor}
\usepackage{subcaption}

\usepackage{geometry}
\usepackage{subcaption}
\usepackage[labelfont=bf, font=small]{caption}
\usepackage{booktabs} 

\usepackage{tikz}
\usepackage{pgfplots}
\pgfplotsset{compat=1.18}

\usetikzlibrary{
    arrows.meta,
    positioning,
    calc,
    shapes.geometric,
    shapes.misc,
    backgrounds,
    fit,
    patterns,
    shadows,
    fadings
}

\usepackage{amssymb}

\definecolor{acadblue}{RGB}{0, 113, 178} 
\definecolor{acadred}{RGB}{136, 34, 85} 
\definecolor{acadgreen}{RGB}{0, 158, 116} 
\definecolor{acadpurple}{RGB}{129, 114, 179} 
\definecolor{acadgray}{RGB}{105, 105, 105}

\definecolor{nips_red}{RGB}{214, 60, 60}
\definecolor{nips_blue}{RGB}{60, 120, 214}
\definecolor{nips_green}{RGB}{60, 160, 80}
\definecolor{nips_purple}{RGB}{140, 60, 180}
\definecolor{nips_dark}{RGB}{60, 60, 65}
\definecolor{nips_bg}{RGB}{245, 246, 250}

\newcommand{\legendbox}[2]{%
    \begingroup
    \setlength{\fboxsep}{1.5pt}
    \smash{\colorbox{#1}{\strut #2}}%
    \endgroup
}

\makeatletter
\newcommand{\printfnsymbol}[1]{%
  \textsuperscript{\@fnsymbol{#1}}%
}
\makeatother

\usepackage{algorithm}
\usepackage{algpseudocode}

\usepackage{amsmath}
\usepackage{amssymb}

\begin{document}

\title{FederatedFactory: Generative One-Shot Learning for Extremely Non-IID Distributed Scenarios}

\titlerunning{FederatedFactory}

\author{Andrea Moleri\inst{1, 2}\textsuperscript{*}\and
Christian Internò\inst{3}\textsuperscript{*} \and Ali Raza\inst{1} \and Markus Olhofer\inst{1} \and \\
David Klindt\inst{4} \and Fabio Stella\inst{2}\textsuperscript{\dag} \and 
Barbara Hammer\inst{3}\textsuperscript{\dag}}

\authorrunning{A.~Moleri et al.}

\institute{Honda Research Institute Europe, Germany \and University of Milan-Bicocca, Italy \and Bielefeld University, Germany \and Cold Spring Harbor Laboratory, U.S.}

\maketitle

\begingroup
  \renewcommand{\thefootnote}{\relax} 
  \footnotetext{\textsc{Correspondence:} \email{a.moleri@campus.unimib.it; christian.interno@uni-bielefeld.de}}
\footnotetext{\textsc{Code}: https://github.com/andreamoleri/FederatedFactory}

\endgroup

\vspace{-1.8em} 
\begin{center}
  \small{\textsuperscript{*} \textit{Equal contribution.} \qquad \textsuperscript{\dag} \textit{Co-advised.}}
\end{center}
\begin{abstract}
Federated Learning (\textsc{FL}) enables distributed optimization without compromising data sovereignty. Yet, where local label distributions are mutually exclusive, standard weight aggregation fails due to conflicting optimization trajectories. Often, FL methods rely on pretrained foundation models, introducing unrealistic assumptions. We introduce \textsc{FederatedFactory}, a zero-dependency framework that inverts the unit of federation from discriminative parameters to generative priors. By exchanging generative modules in a single communication round, our architecture supports \textit{ex nihilo} synthesis of universally class-balanced datasets, eliminating gradient conflict and external prior bias entirely. Evaluations across diverse medical imagery benchmarks, including \textsc{MedMNIST} and \textsc{ISIC2019}, demonstrate that our approach recovers centralized upper-bound performance. Under pathological heterogeneity, it lifts baseline accuracy from a collapsed $11.36\%$ to $90.57\%$ on \textsc{CIFAR-10} and restores \textsc{ISIC2019} AUROC to $90.57\%$. Additionally, this framework facilitates exact modular unlearning through the deterministic deletion of specific generative modules.
\end{abstract}


\section{Introduction}
\label{sec:introduction}
FL provides a decentralized framework to optimize statistical models across $K$ distinct clients while strictly preserving local data sovereignty \cite{McMahan2017FedAvg}. However, the theoretical convergence guarantees of traditional \textsc{FL} rely on independently and identically distributed (IID) data, an assumption violated by extreme statistical heterogeneity in real-world applications like multi-institutional medical imaging \cite{Shelleretal2020, Riekeetal2020}. The structural fragility of standard parameter aggregation is exposed most aggressively under pathological label skew \cite{Hsu2019Dirichlet, JamaliRadetal2021}, where each client holds samples from only a few, or even just one, class of data. Specifically, we define and analyze the extreme \emph{single-class silo} regime, where each client $k \in \{1, \dots, K\}$ possesses a local dataset $\mathcal{D}_k$ that contains exclusively one class, yielding mutually disjoint label supports $\mathcal{Y}_k$. Optimizing shared discriminative parameters $\mathbf{w}$ fails unconditionally in this setting. Lacking the counterfactual data necessary to form an inter-class decision boundary, the local empirical risks $\mathcal{L}_k(\mathbf{w})$ produces optimization trajectories that interfere with one another.

To avoid communication bottlenecks and conflicts in gradient trajectories of iterative \textsc{FL}, One-Shot Federated Learning (\textsc{OSFL}) attempts to aggregate knowledge in a single round of communication \cite{Guha2019OneShot, Li2021Practical}. Recent generative \textsc{OSFL} methods synthesize datasets using pretrained Foundation Models (\textsc{FM}s) as universal priors \cite{Beitollahi2024Parametric, Zaland2025OneShot}. While effective for general domains, this dependency is unreliable for specialized applications such as medical diagnosis~\cite{Zhang2024Challenges, Ma2024Segment, Huix2024Are}. By projecting a rare target distribution onto an external representation space, these methods effectively discard the off-manifold feature components $\mathbf{x}_{\perp}$ that constitute the true diagnostic signal \cite{He2026Ai}.

We introduce \textsc{FederatedFactory}, a framework that aims to recover the centralized upper-bound performance under the extreme \textit{single-class silo} assumption via a novel, strictly zero-dependency Generative \textsc{OSFL} architecture. We invert the unit of federation from discriminative parameter matrices $\mathbf{W}$ to localized generative prior parameters $\boldsymbol{\theta}_k$. Each client independently trains and transmits a generative Factory $G_{\boldsymbol{\theta}_k}$ exactly once, supporting both a centralized architecture (in consortiums where a central aggregator can be trusted) and a fully decentralized Peer-to-Peer (P2P) network mesh (in consortiums in which a central aggregator cannot be trusted). Specifically to our work, we use the computationally efficient EDM2 diffusion model \cite{Karras2022EDM, Karras2024EDM2}.

The server concatenates these decoupled generative structures into a universal prior, synthesizing class-balanced datasets \emph{ex nihilo} from a standard latent space $\mathcal{Z}$. By relying exclusively on models trained directly on the true localized data distributions, \textsc{FederatedFactory} explicitly avoids projection errors and eliminates the external prior bias inherent to \textsc{FM}-reliant methods.

\begin{hypothesis}
We hypothesize that shifting the unit of federation from discriminative parameters to generative priors enables \textsc{OSFL} to achieve performance comparable to centralized baselines in pathologically non-IID scenarios, strictly without raw data exchange or reliance on external \textsc{FM}s.
\end{hypothesis}
 We summarize our contributions as follows:
\begin{itemize}
    \item \textbf{Robustness to Extreme Heterogeneity:} \textsc{FederatedFactory} recovers centralized performance under pathological single-class silos where standard methods collapse (e.g, \textsc{CIFAR10} \cite{Krizhevsky2009CIFAR} Accuracy $11.36\% \rightarrow 90.57\%$, and \textsc{ISIC2019} \cite{Terrail2022FLamby} AUROC $47.31\% \rightarrow 90.57\%$).
    \item \textbf{Zero-Dependency Federation:} By decoupling data synthesis from external pre-trained \textsc{FM}s, our protocol relies exclusively on localized priors. As formalized in \cref{theorem:convergence}, this bounds the global risk strictly by the local generative error ($\bar{\epsilon}$).
    \item \textbf{One-Shot Communication Efficiency:} \textsc{FederatedFactory} relies on a single communication round ($\mathcal{C}_{\mathrm{rounds}} = 1$), avoiding expensive iterations.
    \item \textbf{Modular Machine Unlearning:} The framework guarantees exact modular unlearning~\cite{Geirhos2025Towards}. Removing a client requires only the structural deletion of their corresponding parameter coordinates ($\boldsymbol{\Gamma}_{:, k} \leftarrow \emptyset$).
\end{itemize}

\section{Related Work}
\label{sec:related_work}



\textbf{Optimization Under Distribution Shift.} Federated methods (e.g., \textsc{FedDyn}~\cite{Acar2021FedDyn}, \textsc{FedProx}~\cite{Li2020FedProx}, \textsc{SCAFFOLD}~\cite{Karimireddy2020SCAFFOLD}) handle statistical heterogeneity by bounding local updates. However, they fail when local label sets are entirely disjoint ($\mathcal{Y}_i \cap \mathcal{Y}_j = \emptyset$). Despite strong collaboration incentives in this extreme regime (e.g., isolated hospitals needing generalized models), disjoint labels cause actively divergent gradients \cite{Zhao2018NonIID}. Without overlapping classes to anchor a shared feature space, proximal constraints cannot align these conflicting trajectories to form coherent inter-class decision boundaries \cite{Li2020FedProx}.

\textbf{One-Shot FL and External Priors.} \textsc{OSFL} bypasses iterative gradient conflicts via a single upstream transmission. Recent diffusion frameworks (e.g., \textsc{FedLMG}~\cite{Yang2024FedLMG}, \textsc{FedSDE}~\cite{Qiu2025FedSDE}, \textsc{FedDEO}~\cite{Yang2024FedDEO}) synthesize datasets centrally but rely heavily on pretrained \textsc{FM}s (e.g., CLIP~\cite{Radford2021LearningTV}, Stable Diffusion~\cite{Blattmann2022Retrieval}). Relying on an external manifold $\mathcal{M}_{\mathrm{FM}}$ inherently projects away rare features $\mathbf{x}_{\perp}$; in medical imaging, this risks texture bias and semantic hallucination~\cite{Geirhos2018ImagenetTrained}. We eliminate this \textsc{FM} dependency by transmitting locally trained generative priors $\boldsymbol{\theta}_k$, achieving \textit{ex nihilo} one-shot synthesis exclusively from true local distributions.
Other approaches collaboratively train generative models to construct synthetic datasets. Notably, the Diffusion Federated Dataset~\cite{Hahn2024Diffusion} models generation as cooperative sampling from diffusion energy-based models. While analytically robust, its reliance on iterative communication rounds increases communication overhead. 

\section{Background and Mathematical Preliminaries}

\paragraph{\textbf{Federated Learning under Single-Class Silo Regime.}}
\label{sec:FLObj}
In standard \textsc{FL} model aggregation, a decentralized network of $K$ clients aims to aggregate local empirical risks over shared discriminative parameters $\mathbf{w} \in \mathbb{R}^m$ in order to effectively approximate the global optimum. The \textit{de facto} standard \textsc{FL} approach~\cite{McMahan2017FedAvg} formally attempts to minimize the global objective function $\label{eq:fedavg}
\min_{\mathbf{w}} \mathcal{L}(\mathbf{w}) := \frac{1}{K} \sum_{k=1}^{K} \mathcal{L}_k(\mathbf{w})$, where $\mathcal{L}_k(\mathbf{w}) = \mathbb{E}_{(x, y) \sim p_k}[\ell(\mathbf{w}; x, y)]$ is the local objective representing the expected loss over the true data distribution $p_k$ at client $k$, and $\ell$ denotes the per-sample loss function. 

While aggregators such as \textsc{FedAvg} \cite{McMahan2017FedAvg} succeed when the local datasets $\mathcal{D}_k$ are IID, this paradigm collapses under severe non-IID settings. We parameterize label skewness via a Dirichlet distribution $\mathrm{Dir}_C(\alpha)$. As $\alpha \to 0$ in a cross-silo setting with $K$ clients and $C$ global classes, we reach the pathological \emph{Single-Class Silo} regime (Figure~\ref{fig:spectrum_distribution}c). Here, each client dataset $\mathcal{D}_k$ contains exactly one unique class, yielding strictly disjoint label sets ($\mathcal{Y}_i \cap \mathcal{Y}_j = \emptyset$ for $i \neq j$). Consequently, local optimization trajectories diverge, causing severe gradient conflict \cite{Yu2020GradientSurgery, Zhao2018NonIID}.
\begin{figure}[t]

    \centering
    \begin{subfigure}[b]{0.32\textwidth}
        \centering
        \begin{tikzpicture}[scale=0.60]
            \begin{axis}[
                ybar, 
                bar width=6pt, 
                enlarge x limits=0.25, 
                width=1.1\textwidth, height=4.5cm,
                symbolic x coords={1, 2, 3},
                xtick=data, ymin=0, ymax=1.0,
                ylabel={Density},
                title={\footnotesize \textbf{(a) IID / Uniform} ($\alpha \to \infty$)},
                ymajorgrids=true, grid style=dashed,
                legend style={at={(0.5,-0.3)}, anchor=north, font=\tiny}
            ]
                \addplot[fill=acadred]   coordinates {(1,0.33) (2,0.33) (3,0.33)};
                \addplot[fill=acadblue]  coordinates {(1,0.33) (2,0.33) (3,0.33)};
                \addplot[fill=acadgreen] coordinates {(1,0.33) (2,0.33) (3,0.33)};
            \end{axis}
        \end{tikzpicture}
    \end{subfigure}
    \hfill
    \begin{subfigure}[b]{0.32\textwidth}
        \centering
        \begin{tikzpicture}[scale=0.60]
            \begin{axis}[
                ybar, 
                bar width=6pt, 
                enlarge x limits=0.25, 
                width=1.1\textwidth, height=4.5cm,
                symbolic x coords={1, 2, 3},
                xtick=data, ymin=0, ymax=1.0,
                title={\footnotesize \textbf{(b) Dirichlet Skew} ($\alpha=0.5$)},
                ymajorgrids=true, grid style=dashed,
                legend style={at={(0.5,-0.3)}, anchor=north, font=\tiny}
            ]
                \addplot[fill=acadred]   coordinates {(1,0.6) (2,0.2) (3,0.2)};
                \addplot[fill=acadblue]  coordinates {(1,0.1) (2,0.8) (3,0.1)};
                \addplot[fill=acadgreen] coordinates {(1,0.2) (2,0.1) (3,0.7)};
            \end{axis}
        \end{tikzpicture}
    \end{subfigure}
    \hfill
    \begin{subfigure}[b]{0.32\textwidth}
        \centering
        \begin{tikzpicture}[scale=0.60]
            \begin{axis}[
                ybar stacked, 
                bar width=6pt, 
                enlarge x limits=0.25, 
                width=1.1\textwidth, height=4.5cm,
                symbolic x coords={1, 2, 3},
                xtick=data, ymin=0, ymax=1.0,
                title={\footnotesize \textbf{(c) Single-Class Silo} ($\alpha \to 0$)},
                ymajorgrids=true, grid style=dashed,
                legend style={at={(0.5,-0.3)}, anchor=north, font=\tiny}
            ]
                \addplot[fill=acadred]   coordinates {(1,1.0) (2,0.0) (3,0.0)};
                \addplot[fill=acadblue]  coordinates {(1,0.0) (2,1.0) (3,0.0)};
                \addplot[fill=acadgreen] coordinates {(1,0.0) (2,0.0) (3,1.0)};
            \end{axis}
        \end{tikzpicture}
    \end{subfigure}
    
    \caption{\textbf{The Spectrum of Heterogeneity.} \textbf{(a)} Ideal IID data (uniform overlap). \textbf{(b)} Dirichlet-distributed skew (imbalanced overlap). \textbf{(c)} Single-Class Silo (pathologically disjoint supports), representing the extreme theoretical limit.}
    \label{fig:spectrum_distribution}

    \begin{takeaway}
        Under the \textit{single-class silo regime}, standard FL collapses. This extreme label skew induces gradient conflict across clients, rendering standard parameter aggregation incapable of forming a coherent global decision boundary.
    \end{takeaway}
    
\end{figure}
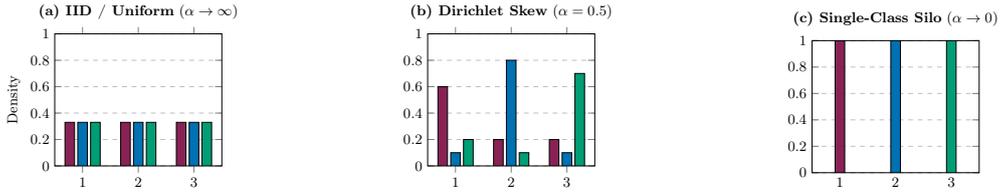


\paragraph{\textbf{FederatedFactory Formal Problem Statement.}}
\label{sec:Psett}
Our objectives challenge the standard setting defined in \cref{eq:fedavg}. We operate strictly within the \textit{cross-silo} regime, which stands in contrast to the more common \textit{cross-device} setting. While \textit{cross-device} learning involves massive populations of transient, resource-constrained mobile or IoT units, the \textit{cross-silo} regime is characterized by a small, fixed number of stakeholders (such as hospitals or financial institutions) possessing localized, high-capacity computational resources and persistent connectivity~\cite{Kairouz2021Advances}. We aim to design a \textsc{FL} optimization protocol $\mathcal{A}$ that recovers the optimal centralized decision boundary while operating under severe decentralized restrictions. Let $\mathcal{D}_{\mathrm{union}} = \bigcup_{k=1}^K \mathcal{D}_k$ denote the inaccessible theoretical centralized dataset, and let $\mathbf{w}^* = \arg\min_{\mathbf{w}} \mathbb{E}_{(x,y) \sim \mathcal{D}_{\mathrm{union}}}[\ell(\mathbf{w}, x, y)]$ represent the optimal parameters obtained via centralized training. We model the federated protocol $\mathcal{A}$ as a set of local client functions $f_k$ generating upstream messages $M_k = f_k(\mathcal{D}_k)$, and a server aggregation function $g$ such that the final synthesized model has parameters $\mathbf{w}_{\mathcal{A}} = g(M_1, \dots, M_K)$. We formalize this objective as a constrained optimization problem seeking to minimize the excess global risk:
\begin{equation}
\begin{aligned}
\min_{f_1, \dots, f_K, g} \quad & \left| \mathbb{E}_{(x,y) \sim \mathcal{D}_{\mathrm{union}}} \left[ \ell(\mathbf{w}_{\mathcal{A}}, x, y) \right] - \mathbb{E}_{(x,y) \sim \mathcal{D}_{\mathrm{union}}} \left[ \ell(\mathbf{w}^*, x, y) \right] \right| \\
\text{s.t.} \quad 
& \textbf{C1} \text{ (Pathological Skew):} \;\; \text{supp}(p_i(y)) \cap \text{supp}(p_j(y)) = \emptyset, \quad \forall i \neq j \\
& \textbf{C2} \text{ (Zero-Dependency):} \;\; \mathbf{w}_{\mathcal{A}} \text{ without external prior } \boldsymbol{\theta}_{\mathrm{FM}} \\
& \textbf{C3} \text{ (Strict Sovereignty):} \;\; x \notin M_k, \quad \forall x \in \mathcal{D}_k \\
& \textbf{C4} \text{ (One-Shot Comm.):} \;\; \mathcal{C}_{\text{rounds}} = 1
\end{aligned}
\end{equation}
Here, $\ell$ is the per-sample task-specific loss function. \textbf{C1} mandates convergence in the extreme limit of the Dirichlet distribution ($\alpha \to 0$), where $p_i(y)$ denotes the marginal label distribution at client $i$. In this pathological regime, discriminative gradient trajectories are actively opposed due to completely disjoint label supports. \textbf{C2} prohibits reliance on external foundation models ($\mathcal{M}_{\mathrm{FM}}$) or public proxy datasets to synthesize missing counterfactuals, isolating the framework from external prior bias. \textbf{C3} guarantees no raw samples are transmitted in the uplink communication messages $M_k$. Finally, \textbf{C4} restricts the system to exactly one asynchronous upstream communication per client (where $\mathcal{C}_{\text{rounds}}$ is the total number of communications), permanently eliminating iterative communication overhead and bidirectional synchronization requirements.

\subsection{Theoretical Guarantee of Zero-Dependency Convergence}
\label{sec:theoretical_motivation}

Recent generative \textsc{OSFL} frameworks ensure global convergence using server-side Foundation Models (\textsc{FM}s)~\cite{Yang2024FedDEO, Yang2024FedLMG, Qiu2025FedSDE}, assuming the \textsc{FM}'s distribution $p_{\mathrm{FM}}(\mathbf{x})$ sufficiently overlaps with local client data $p_k(\mathbf{x})$. This overlap is bounded by $\lambda$, which quantifies the maximum OOD penalty between the pre-trained manifold and private data. While effective for natural images, this assumption fails in specialized modalities (e.g, clinical) where severe domain shifts cause $\lambda \to \infty$.

To motivate \textsc{FederatedFactory}, we establish a convergence guarantee that completely bypasses this \textsc{FM} overlap assumption. Let $p_{\mathrm{union}}(\mathbf{x}, y)$ be the inaccessible true global joint distribution. Under the Single-Class Silo constraint (\textbf{C1}) with disjoint label supports, $p_{\mathrm{union}}$ reduces to a strict mixture of local marginals $p_k(\mathbf{x})$ weighted by empirical proportions $\pi_k = |\mathcal{D}_k| / |\mathcal{D}_{\mathrm{union}}|$:
\begin{equation}
    p_{\mathrm{union}}(\mathbf{x}, y) = \sum_{k=1}^K \pi_k p_k(\mathbf{x}) \mathbb{I}(y = y_k)
\end{equation}
where $\mathbb{I}(\cdot)$ is the indicator function isolating the localized class.

By transmitting only the localized generative prior $\boldsymbol{\theta}_k$ (\textbf{C3}, \textbf{C4}), the server constructs a fully synthetic global distribution $\hat{p}_{\mathrm{syn}}(\mathbf{x}, y)$ \textit{ex nihilo}, complying with the zero-dependency constraint (\textbf{C2}):
\begin{equation}
    \hat{p}_{\mathrm{syn}}(\mathbf{x}, y) = \sum_{k=1}^K \pi_k p_{\boldsymbol{\theta}_k}(\mathbf{x}) \mathbb{I}(y = y_k)
\end{equation}

To prove convergence without external priors, we establish two standard assumptions:

\textbf{Assumption 1 (Local Generative Convergence).} \textit{The local diffusion training objective (ELBO) directly minimizes the Kullback-Leibler (KL) divergence. We assume this local optimization error is bounded by $\epsilon_k$ for all clients $k \in \{1, \dots, K\}$:}
\begin{equation}
    \label{eq:local_kl}
    \mathrm{KL}(p_k(\mathbf{x}) \parallel p_{\boldsymbol{\theta}_k}(\mathbf{x})) \le \epsilon_k
\end{equation}

\textbf{Assumption 2 (Bounded Risk Function).} \textit{The task-specific per-sample loss function $\ell(\mathbf{w}, \mathbf{x}, y)$ is bounded by a constant $M > 0$, such that $\sup_{\mathbf{w}, \mathbf{x}, y} \ell(\mathbf{w}, \mathbf{x}, y) \le M$.}

\begin{mdframed}[topline=false, bottomline=false, rightline=false, leftline=true, linewidth=2pt, linecolor=geoPrimary, innerleftmargin=8pt, innerrightmargin=0pt, innertopmargin=0pt, innerbottommargin=0pt, skipabove=10pt, skipbelow=10pt]
\begin{lemma}[Global Manifold Recovery]
\label{lemma:manifold_recovery}
Under constraint \textbf{C1} and Assumption 1, the KL divergence between the true global distribution and the zero-dependency synthetic distribution is strictly bounded by the weighted sum of the local diffusion errors: $\mathrm{KL}(p_{\mathrm{union}} \parallel \hat{p}_{\mathrm{syn}}) \le \sum_{k=1}^K \pi_k \epsilon_k$.
\end{lemma}
\end{mdframed}

\begin{proof}
Based on the fundamental definition of the joint KL divergence, we expand the integral:
\begin{equation}
    \mathrm{KL}(p_{\mathrm{union}} \parallel \hat{p}_{\mathrm{syn}}) = \sum_{y} \int p_{\mathrm{union}}(\mathbf{x}, y) \log \frac{p_{\mathrm{union}}(\mathbf{x}, y)}{\hat{p}_{\mathrm{syn}}(\mathbf{x}, y)} d\mathbf{x}
\end{equation}
Because the label supports are completely disjoint (\textbf{C1}), for any given class $y_k$, the cross-terms of the mixture evaluate exactly to zero. The summation over $y$ thus collapses perfectly to the individual client indices $k$. Substituting the marginal definitions:
\begin{equation}
    = \sum_{k=1}^K \int \Big( \pi_k p_k(\mathbf{x}) \Big) \log \frac{\pi_k p_k(\mathbf{x})}{\pi_k p_{\boldsymbol{\theta}_k}(\mathbf{x})} d\mathbf{x} = \sum_{k=1}^K \pi_k \mathrm{KL}(p_k(\mathbf{x}) \parallel p_{\boldsymbol{\theta}_k}(\mathbf{x}))
\end{equation}
Substituting the local convergence bound from Assumption 1 yields the final inequality. $\hfill \blacksquare$
\end{proof}

We define $\mathcal{L}_{\mathrm{true}}(\mathbf{w}) = \mathbb{E}_{(\mathbf{x},y) \sim p_{\mathrm{union}}}[\ell(\mathbf{w}, \mathbf{x}, y)]$ as the true centralized risk, and $\mathcal{L}_{\mathrm{syn}}(\mathbf{w}) = \mathbb{E}_{(\mathbf{x},y) \sim \hat{p}_{\mathrm{syn}}}[\ell(\mathbf{w}, \mathbf{x}, y)]$ as the surrogate risk evaluated on the synthesized dataset. Let $\bar{\epsilon} = \sqrt{\frac{1}{2} \sum_{k=1}^K \pi_k \epsilon_k}$ denote the aggregate generative error mapped to the Total Variation (TV) space.

\begin{mdframed}[topline=false, bottomline=false, rightline=false, leftline=true, linewidth=2pt, linecolor=geoPrimary, innerleftmargin=8pt, innerrightmargin=0pt, innertopmargin=0pt, innerbottommargin=0pt, skipabove=10pt, skipbelow=10pt]
\begin{theorem}[Zero-Dependency Aggregation]
\label{theorem:convergence}
Under Assumptions 1 and 2, the excess global risk of the classifier $\mathbf{w}_{\mathcal{A}} = \arg\min_{\mathbf{w}} \mathcal{L}_{\mathrm{syn}}(\mathbf{w})$ trained exclusively on the generated synthetic distribution, compared to the optimal centralized classifier $\mathbf{w}^* = \arg\min_{\mathbf{w}} \mathcal{L}_{\mathrm{true}}(\mathbf{w})$, is strictly bounded by:

\vspace{0.2cm}
\begin{minipage}{0.55\textwidth}
    \begin{equation}
        \underbrace{\mathcal{L}_{\mathrm{true}}(\mathbf{w}_{\mathcal{A}})}_{\text{Federated Classifier}} - \underbrace{\mathcal{L}_{\mathrm{true}}(\mathbf{w}^*)}_{\text{Centralized Classifier}} \le \underbrace{2M\bar{\epsilon}}_{\text{Max Penalty}}
    \end{equation}
\end{minipage}%
\begin{minipage}{0.45\textwidth}
    \centering
    \begin{tikzpicture}[scale=0.90, every node/.style={scale=0.85}]
        \draw[->, gray!80] (-1.3, 0) -- (1.4, 0) node[right, font=\scriptsize, text=black] {$\mathbf{w}$};
        \draw[->, gray!80] (0, -0.2) -- (0, 2.2) node[above, font=\scriptsize, text=black] {Risk $\mathcal{L}$};
        
        \draw[thick, acadblue, domain=-1.1:1.1, samples=50] plot (\x, {1.5*\x*\x + 0.2});
        \node[text=acadblue, font=\scriptsize\bfseries] at (-0.8, 1.8) {$\mathcal{L}_{\mathrm{true}}$};
        
        \draw[thick, dashed, acadred, domain=-0.2:1.3, samples=50] plot (\x, {1.5*(\x-0.6)*(\x-0.6) + 0.4});
        \node[text=acadred, font=\scriptsize\bfseries] at (1.05, 1.8) {$\mathcal{L}_{\mathrm{syn}}$};
        
        \fill[acadblue] (0, 0.2) circle (1.5pt) node[below left, font=\scriptsize, text=black] {$\mathbf{w}^*$};
        \fill[acadred] (0.6, 0.4) circle (1.5pt) node[below, font=\scriptsize, text=black, yshift=-1pt] {$\mathbf{w}_{\mathcal{A}}$};
        
        \fill[acadblue] (0.6, 0.74) circle (1.5pt);
        \draw[dotted, gray!80, thick] (0.6, 0.4) -- (0.6, 0.74); 
        
        \draw[<->, thick, acadgreen] (1.1, 0.2) -- (1.1, 0.74) node[midway, right, font=\scriptsize\bfseries] {$\le 2M\bar{\epsilon}$};
        \draw[dotted, gray!80, thick] (0, 0.2) -- (1.1, 0.2);
        \draw[dotted, gray!80, thick] (0.6, 0.74) -- (1.1, 0.74);
    \end{tikzpicture}
\end{minipage}
\vspace{0.1cm}
\end{theorem}
\end{mdframed}

\begin{proof}
By the integral definition of the Total Variation (TV) distance and Pinsker's inequality, the expected risk deviation for any arbitrary classifier $\mathbf{w}$ is strictly bounded by:
\begin{equation}
    \left| \mathcal{L}_{\mathrm{true}}(\mathbf{w}) - \mathcal{L}_{\mathrm{syn}}(\mathbf{w}) \right| \le M \cdot \mathrm{TV}(p_{\mathrm{union}} \parallel \hat{p}_{\mathrm{syn}}) \le M \sqrt{\frac{1}{2} \mathrm{KL}(p_{\mathrm{union}} \parallel \hat{p}_{\mathrm{syn}})} \le M\bar{\epsilon}
\end{equation}
Because the federated classifier $\mathbf{w}_{\mathcal{A}}$ is optimized to minimize the synthetic risk, it holds that $\mathcal{L}_{\mathrm{syn}}(\mathbf{w}_{\mathcal{A}}) \le \mathcal{L}_{\mathrm{syn}}(\mathbf{w}^*)$. Applying the bound $M\bar{\epsilon}$ symmetrically transitions between true and synthetic risks:
\begin{equation}
    \mathcal{L}_{\mathrm{true}}(\mathbf{w}_{\mathcal{A}}) \le \mathcal{L}_{\mathrm{syn}}(\mathbf{w}_{\mathcal{A}}) + M\bar{\epsilon} \le \mathcal{L}_{\mathrm{syn}}(\mathbf{w}^*) + M\bar{\epsilon} \le \mathcal{L}_{\mathrm{true}}(\mathbf{w}^*) + 2M\bar{\epsilon}
\end{equation}
which concludes the proof. $\hfill \blacksquare$
\end{proof}


\textbf{Remark.} \cref{theorem:convergence} formally demonstrates that global convergence is achievable under constraints \textbf{C1}--\textbf{C4}. By exchanging generative priors, the global distribution approximation error is entirely bounded by $\epsilon_k$. Unlike \textsc{FM}-dependent methods where the bound is dominated by a rigid, often infinite projection error ($\lambda$) caused by out-of-distribution local data, \textsc{FederatedFactory}'s error approaches zero simply by training the local models to standard convergence.

\section{Methodology}
\label{sec:methodology}

To satisfy the optimization objective outlined in \cref{sec:Psett}, \textsc{FederatedFactory} abandons discriminative parameter aggregation entirely. Under the pathological single-class silo regime (\textbf{C1}), gradient trajectories are in conflict~\cite{Yu2020GradientSurgery}. We resolve this by replacing the unit of communication from discriminative gradients to localized generative prior parameters, establishing a zero-dependency (\textbf{C2}), no-data sharing (\textbf{C3}), one-shot framework (\textbf{C4}) for distributed synthesis.

\subsection{FederatedFactory}
Instead of transmitting parameter updates, which actively conflict across disjoint label spaces~\cite{Yu2020GradientSurgery}, \textsc{FederatedFactory} relies on the transmission of generative model parameters. We define the \emph{Factory} as a localized, self-contained generative module. Each client $k \in \{1, \dots, K\}$ independently optimizes a Factory on its private dataset $\mathcal{D}_k$. While our framework is architecture-agnostic, we specifically instantiate the Factory using the Score-based Diffusion model EDM2 \cite{Karras2024EDM2}. 

In this context, the ``generative blueprint'' consists of the denoising function $G_{\boldsymbol{\theta}_k}$. Unlike traditional Autoencoders~\cite{Davidson2018HypersphericalVAE}, diffusion models do not utilize a deterministic encoder to compress data into a latent bottleneck. Instead, the reverse denoising process acts as the fundamental mapping from a standard normal latent space $\mathcal{Z} \sim \mathcal{N}(\mathbf{0}, \mathbf{I})$ to the learned local manifold $\hat{\mathcal{M}}_k$ \cite{Ho2020DDPM}.


While we empirically focus on Diffusion for its high-fidelity clinical synthesis, \textsc{FederatedFactory} natively supports architectures with explicit encoder-decoder splits (e.g., VAEs~\cite{Kingma2013VAE}) or adversarial mappings (GANs~\cite{Goodfellow2014GAN}), where the shared parameters $\boldsymbol{\theta}_k$ would correspond to the physical decoder or generator. Reducing the communication payload to independent Factories enables flexible global synthesis. Specifically, we instantiate this framework through two distinct operational modes: \textbf{(A)} a centralized architecture designed for consortiums equipped with a trusted aggregator (\cref{subsec:centralized_protocol}), and \textbf{(B)} a fully decentralized peer-to-peer (\textsc{P2P}) mesh optimized for environments in which a centralized entity cannot be trusted (\cref{subsec:decentralized_protocol}).

\begin{takeaway}
By transmitting generative Factories, \textsc{FederatedFactory} enables a blueprint-based \textit{ex nihilo} generation of class-balanced datasets without relying on overlapping data supports or external \textsc{FM}s.
\end{takeaway}

\paragraph{\textbf{Protocol A: Centralized Synthesis.}}
\label{subsec:centralized_protocol}
Under this configuration (Figure~\ref{fig:centralized_workflow}), the framework designates the central aggregator as the point of data generation, restricting network overhead to exactly one upstream transmission per client ($\mathcal{C}_{\text{rounds}} = 1$).
The server aggregates the generative mappings into a unified library $\boldsymbol{\Theta} = \{G_{\boldsymbol{\theta}_1}, \dots, G_{\boldsymbol{\theta}_K}\}$. Exploiting the universality of the standard normal latent space $\mathcal{Z}$, the server samples noise vectors $\mathbf{z} \sim \mathcal{N}(\mathbf{0}, \mathbf{I})$ and projects them through the respective Factory. This materializes a fully synthetic, class-balanced global dataset $\hat{\mathcal{D}}_{\mathrm{syn}}$ \textit{ex nihilo}, with an arbitrarily large number of samples whose diversity is bounded by the entropy of the localized generative priors:
\begin{equation}
\hat{\mathcal{D}}_{\mathrm{syn}} = \bigcup_{k=1}^K \left\{ (\hat{\mathbf{x}}, y_k) \mid \hat{\mathbf{x}} = G_{\boldsymbol{\theta}_k}(\mathbf{z}), \mathbf{z} \sim \mathcal{N}(\mathbf{0}, \mathbf{I}) \right\}
\end{equation}
We then optimize a global classifier $\mathbf{w}$ (ResNet-50~\cite{He2016ResNet}) exclusively on $\hat{\mathcal{D}}_{\mathrm{syn}}$. We adopt this model to isolate the improvements of our generative framework without introducing architectural bottlenecks. We reformulate the standard FL parameter aggregation into an optimization over a global surrogate objective:
\begin{equation}
\min_{\mathbf{w}} \frac{1}{K} \sum_{k=1}^K \mathbb{E}_{\mathbf{z} \sim \mathcal{N}(\mathbf{0}, \mathbf{I})} \left[ \ell \left( \mathbf{w}; G_{\boldsymbol{\theta}_k}(\mathbf{z}), y_k \right) \right]
\end{equation}
Because this synthesis relies exclusively on localized generative priors optimized directly on the true data manifolds, \textsc{FederatedFactory} ensures that the generated distribution includes the rare diagnostic support (as required in \cref{sec:theoretical_motivation}).

\begin{figure}[htpb]
    \centering
    \definecolor{nips_red}{RGB}{136, 34, 85}
    \definecolor{nips_blue}{RGB}{0, 113, 178}
    \definecolor{nips_green}{RGB}{0, 158, 116}
    \definecolor{nips_purple}{RGB}{136, 136, 136} 
    \definecolor{nips_dark}{RGB}{60, 60, 65}
    \definecolor{nips_bg}{RGB}{245, 246, 250}

    \resizebox{\textwidth}{!}{%
    \begin{tikzpicture}[
        font=\sffamily,
        >=LaTeX, 
        node distance=1.0cm and 1.5cm,
        silo/.style={
            draw=gray!20, fill=gray!5, rounded corners=8pt, 
            inner sep=8pt, dashed, line width=0.8pt
        },
        server_box/.style={
            draw=gray!30, fill=white, rounded corners=15pt, 
            drop shadow={opacity=0.15, shadow xshift=2pt, shadow yshift=-2pt}, 
            line width=1.0pt
        },
        database/.style={
            cylinder, shape border rotate=90, aspect=0.25,
            draw=gray!60, line width=0.8pt,
            minimum height=1.0cm, minimum width=0.85cm, 
            font=\scriptsize\bfseries, align=center,
            top color=white, bottom color=gray!10 
        },
        factory/.style={
            rectangle, draw=nips_dark!80, line width=0.8pt, rounded corners=4pt,
            fill=white, minimum width=1.6cm, minimum height=0.9cm, 
            font=\bfseries\small, align=center, drop shadow={opacity=0.08}
        },
        decoder/.style={
            factory, fill=white, dashed, font=\small\itshape, line width=0.8pt
        },
        noise_op/.style={
            circle, draw=nips_purple, thick, fill=white, 
            inner sep=2pt, minimum size=1.0cm,
            font=\small, align=center, dashed
        },
        flow/.style={->, thick, draw=nips_dark!80},
        upload/.style={->, thick, dashed, draw=nips_dark!80, shorten >=2pt, shorten <=2pt},
        math_label/.style={midway, above, font=\tiny, text=gray!90}
    ]

        
        \node[database, draw=nips_red, bottom color=nips_red!10] (RealA) at (0, 3.5) {Real\\$\mathcal{D}_1$};
        \node[factory, right=1.4cm of RealA] (FactA) {Factory\\A};
        \draw[flow] (RealA) -- node[math_label] {$x \sim p_1(x)$} (FactA);
        
        \begin{scope}[on background layer]
            \node[silo, fit=(RealA) (FactA), label={[text=nips_red, font=\bfseries\small]north:Client A (Class 1)}] (SiloA) {};
        \end{scope}

        \node[database, draw=nips_blue, bottom color=nips_blue!10] (RealB) at (0, 0) {Real\\$\mathcal{D}_2$};
        \node[factory, right=1.4cm of RealB] (FactB) {Factory\\B};
        \draw[flow] (RealB) -- node[math_label] {$x \sim p_2(x)$} (FactB);
        
        \begin{scope}[on background layer]
            \node[silo, fit=(RealB) (FactB), label={[text=nips_blue, font=\bfseries\small]north:Client B (Class 2)}] (SiloB) {};
        \end{scope}

        \node[database, draw=nips_green, bottom color=nips_green!10] (RealC) at (0, -3.5) {Real\\$\mathcal{D}_3$};
        \node[factory, right=1.4cm of RealC] (FactC) {Factory\\C};
        \draw[flow] (RealC) -- node[math_label] {$x \sim p_3(x)$} (FactC);
        
        \begin{scope}[on background layer]
            \node[silo, fit=(RealC) (FactC), label={[text=nips_green, font=\bfseries\small]north:Client C (Class 3)}] (SiloC) {};
        \end{scope}

        \node[below=0.1cm of SiloC, font=\scriptsize\color{gray}] {(i) Local Training (e.g. EDM2 \cite{Karras2024EDM2})};


        \node[decoder, draw=nips_red] (DecA) at (6.0, 1.8) {$G_{\theta_1}$};
        \node[decoder, draw=nips_blue] (DecB) at (6.0, 0) {$G_{\theta_2}$};
        \node[decoder, draw=nips_green] (DecC) at (6.0, -1.8) {$G_{\theta_3}$};

        \node[noise_op] (NoiseA) at (8.2, 1.8) {$z\!\sim\!\mathcal{N}$};
        \node[noise_op] (NoiseB) at (8.2, 0) {$z\!\sim\!\mathcal{N}$};
        \node[noise_op] (NoiseC) at (8.2, -1.8) {$z\!\sim\!\mathcal{N}$};

        \node[database, pattern=north east lines, pattern color=nips_red!40, draw=nips_red] (SynA) at (10.5, 1.8) {Syn\\$\hat{\mathcal{D}}_1$};
        \node[database, pattern=north east lines, pattern color=nips_blue!40, draw=nips_blue] (SynB) at (10.5, 0) {Syn\\$\hat{\mathcal{D}}_2$};
        \node[database, pattern=north east lines, pattern color=nips_green!40, draw=nips_green] (SynC) at (10.5, -1.8) {Syn\\$\hat{\mathcal{D}}_3$};

        \draw[flow, dashed] (DecA) -- (NoiseA);
        \draw[flow, dashed] (DecB) -- (NoiseB);
        \draw[flow, dashed] (DecC) -- (NoiseC);
        
        \draw[flow] (NoiseA) -- node[midway, above, font=\tiny, yshift=-1pt] {$\hat{x}=G_{\theta_1}(z)$} (SynA);
        \draw[flow] (NoiseB) -- node[midway, above, font=\tiny, yshift=-1pt] {$\hat{x}=G_{\theta_2}(z)$} (SynB);
        \draw[flow] (NoiseC) -- node[midway, above, font=\tiny, yshift=-1pt] {$\hat{x}=G_{\theta_3}(z)$} (SynC);
        
        \begin{scope}[on background layer]
            \node[server_box, fit=(DecA) (SynA) (SynC) (DecC), inner xsep=12pt, inner ysep=20pt, label={[font=\bfseries]north:Server (Aggregator)}] (ServerFrame) {};
        \end{scope}

        \node[below=0.2cm of ServerFrame, font=\scriptsize\color{gray}] {(ii) Upload \quad (iii) Generative Sampling: $z \to \hat{x}$};


        \draw[upload] (FactA.east) to[out=0, in=180] 
            node[pos=0.4, above, sloped, font=\tiny] {Params $\theta_1$} 
            (DecA.west);
         
        \draw[upload] (FactB.east) to[out=0, in=180] 
            node[pos=0.4, above, sloped, font=\tiny] {Params $\theta_2$} 
            (DecB.west);
         
        \draw[upload] (FactC.east) to[out=0, in=180] 
            node[pos=0.65, above, sloped, font=\tiny] {Params $\theta_3$} 
            (DecC.west);

        \node[rectangle, rounded corners=6pt, draw=nips_dark, line width=1.2pt, fill=white, 
              minimum height=2.5cm, minimum width=2.5cm, align=center,
              drop shadow={opacity=0.2, shadow xshift=3pt}] (Global) at (14.2, 0) {\textbf{Global}\\\textbf{Classifier}};

        \draw[flow] (SynA.east) to[out=0, in=165] (Global.west);
        \draw[flow] (SynB.east) -- (Global.west);
        \draw[flow] (SynC.east) to[out=0, in=195] (Global.west);

        \node[below=0.6cm of Global, font=\small\bfseries] (Inf) {Inference};
        \draw[->, very thick, nips_dark] (Global.south) -- (Inf);

        \node[below=0.3cm of Inf, font=\scriptsize\color{gray}] {(iv) Global Training (e.g. ResNet-50 \cite{He2016ResNet})};
        
    \end{tikzpicture}
    }
    \caption{\textbf{Centralized FederatedFactory Protocol.} Aggregated Factories produce a fully synthetic dataset $\hat{\mathcal{D}}$, enabling a global classifier training without raw data access.}
    \label{fig:centralized_workflow}
\end{figure}
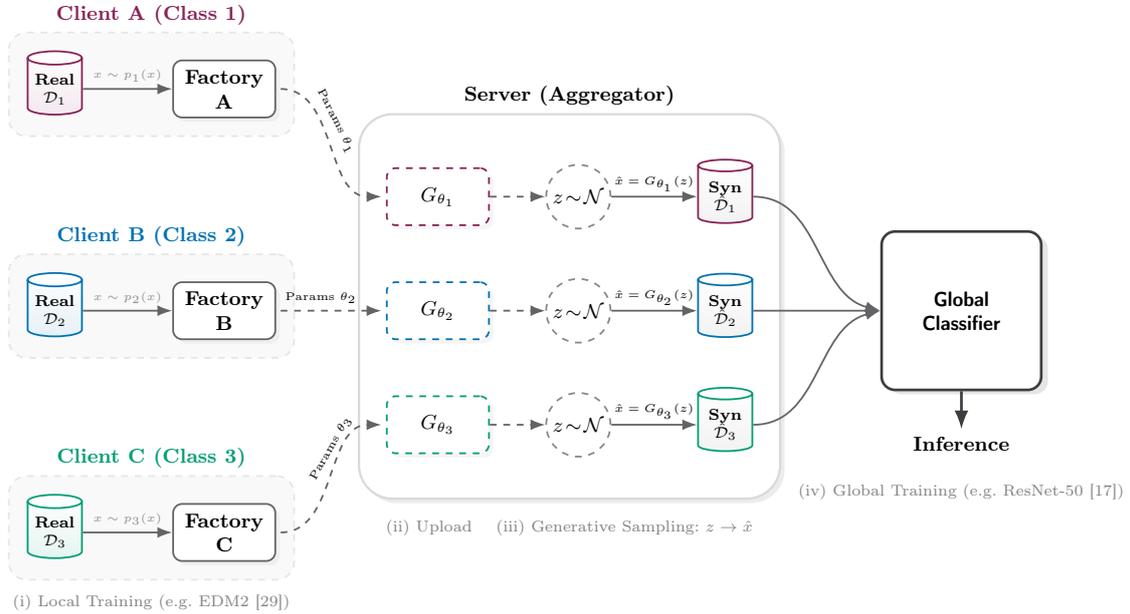

\paragraph{\textbf{Protocol B: Decentralized Synthesis.}}
\label{subsec:decentralized_protocol}
Under this configuration (Figure~\ref{fig:decentralized_workflow}), each client $k$ broadcasts its generative prior $G_{\boldsymbol{\theta}_k}$ to all participating peers. Upon receiving the complement Factories for all disjoint classes $j \neq k$, every client $k$ locally synthesizes the missing distributions, constructing a hybrid dataset $\mathcal{D}_k^{\mathrm{mix}}$ in which local real data $\mathcal{D}_k$ is augmented with synthetic samples $\hat{\mathbf{x}}$:
\begin{equation}
    \mathcal{D}_k^{\mathrm{mix}} = {\mathcal{D}_k} \cup {\bigcup_{j \neq k} \left\{ (\hat{\mathbf{x}}, y_j) \mid \hat{\mathbf{x}} = G_{\boldsymbol{\theta}_j}(\mathbf{z}), \mathbf{z} \sim \mathcal{N}(\mathbf{0}, \mathbf{I}) \right\}}
\end{equation}
Client $k$ then optimizes a local discriminative expert classifier $f_{\mathbf{w}_k}$ exclusively on $\mathcal{D}_k^{\mathrm{mix}}$. At inference time, the distributed mesh aggregates global decisions using a Product of Experts (\textsc{PoE}) formulation \cite{Hinton2002PoE}. Unlike standard mixture-based ensembling, which acts as a logical disjunction and dilutes predictive certainty, \textsc{PoE} functions as a strict intersection of constraints. For a target sample $\mathbf{x}$, the global inference probability is the renormalized product: $p_{\mathrm{PoE}}(y \mid \mathbf{x}) = \frac{1}{Z} \prod_{k=1}^K p_k(y \mid \mathbf{x})$ where $Z$ is the partition function. This enforces a strict consensus veto. To prevent a single overconfident but incorrect expert from indiscriminately zeroing out the consensus, the aggregation operates in log-space with a strict minimum probability floor. Thus, if a local expert assigns a near-zero probability to a spurious feature, the aggregate probability strictly collapses, preserving high-confidence decision boundaries across the distributed network. 

\begin{figure}[htpb]
    \centering
    \definecolor{nips_red}{RGB}{136, 34, 85}
    \definecolor{nips_blue}{RGB}{0, 113, 178}
    \definecolor{nips_green}{RGB}{0, 158, 116}
    \definecolor{nips_purple}{RGB}{136, 136, 136} 
    \definecolor{nips_dark}{RGB}{60, 60, 65}
    \definecolor{nips_bg}{RGB}{245, 246, 250}
    
    \resizebox{\textwidth}{!}{%
    \begin{tikzpicture}[
        font=\sffamily,
        >=LaTeX,
        node distance=0.8cm and 1.2cm,
        silo/.style={
            draw=gray!20, fill=gray!5, rounded corners=8pt, 
            inner sep=8pt, dashed, line width=0.8pt
        },
        database/.style={
            cylinder, shape border rotate=90, aspect=0.25,
            draw=gray!60, line width=0.8pt,
            minimum height=0.8cm, minimum width=0.7cm, 
            font=\scriptsize\bfseries, align=center,
            top color=white, bottom color=gray!10
        },
        factory/.style={
            rectangle, draw=nips_dark!80, line width=0.8pt, rounded corners=4pt,
            fill=white, minimum width=1.4cm, minimum height=0.8cm, 
            font=\bfseries\scriptsize, align=center, drop shadow={opacity=0.08}
        },
        decoder_mini/.style={
            rectangle, draw=gray!50, dashed, rounded corners=2pt,
            fill=white, minimum width=0.9cm, minimum height=0.5cm,
            font=\tiny\itshape, align=center
        },
        model_local/.style={
            rectangle, draw=nips_dark, line width=1.0pt, rounded corners=4pt,
            fill=white, minimum width=1.2cm, minimum height=1.2cm, 
            font=\bfseries\scriptsize, align=center,
            drop shadow={opacity=0.15, shadow xshift=1pt, shadow yshift=-1pt}
        },
        flow/.style={->, thick, draw=nips_dark!80},
        bus_line/.style={thick, rounded corners=4pt},
        dot/.style={circle, fill, inner sep=1.5pt},
        self_data/.style={->, thick, bend right=25, font=\tiny\bfseries},
        math_label/.style={midway, above, font=\tiny, text=gray!90, yshift=-1pt}
    ]

        \node[database, draw=nips_red, bottom color=nips_red!10] (RealA) at (0, 3.5) {Real\\$\mathcal{D}_1$};
        \node[factory, right=1.3cm of RealA] (FactA) {Factory\\A};
        \draw[flow] (RealA) -- node[math_label] {$x \sim p_1(x)$} (FactA);

        \node[decoder_mini, draw=nips_blue, right=3.0cm of FactA, yshift=0.3cm] (DecB_at_A) {$G_{\theta_2}$};
        \node[decoder_mini, draw=nips_green, below=0.1cm of DecB_at_A] (DecC_at_A) {$G_{\theta_3}$};
        \node[database, scale=0.85, right=0.5cm of DecB_at_A, yshift=-0.3cm, pattern=north east lines, pattern color=gray!40] (SynA) {\textbf{$\text{Union } \hat{\mathcal{D}}_A$}\\$Real \cup Syn$};
        \node[model_local, right=0.8cm of SynA] (ModelA) {Model\\$\textbf{w}_A$};
        \draw[flow] (SynA) -- (ModelA);
        
        \draw[->, dashed, gray] (DecB_at_A) -- (SynA);
        \draw[->, dashed, gray] (DecC_at_A) -- (SynA);
        \draw[self_data, draw=nips_red] (RealA.south east) to node[midway, below] {Self $\mathcal{D}_1$} (SynA.south west);

        \begin{scope}[on background layer]
            \node[silo, fit=(RealA) (ModelA), label={[text=nips_red, font=\bfseries\small, anchor=south west]north west:Client A (Class 1)}] {};
        \end{scope}

        \node[database, draw=nips_blue, bottom color=nips_blue!10] (RealB) at (0, 0) {Real\\$\mathcal{D}_2$};
        \node[factory, right=1.3cm of RealB] (FactB) {Factory\\B};
        \draw[flow] (RealB) -- node[math_label] {$x \sim p_2(x)$} (FactB);

        \node[decoder_mini, draw=nips_red, right=3.0cm of FactB, yshift=0.3cm] (DecA_at_B) {$G_{\theta_1}$};
        \node[decoder_mini, draw=nips_green, below=0.1cm of DecA_at_B] (DecC_at_B) {$G_{\theta_3}$};
        \node[database, scale=0.85, right=0.5cm of DecA_at_B, yshift=-0.3cm, pattern=north east lines, pattern color=gray!40] (SynB) {$\text{Union } \hat{\mathcal{D}}_B$\\$Real \cup Syn$};
        \node[model_local, right=0.8cm of SynB] (ModelB) {Model\\$\textbf{w}_B$};
        \draw[flow] (SynB) -- (ModelB);
        
        \draw[->, dashed, gray] (DecA_at_B) -- (SynB);
        \draw[->, dashed, gray] (DecC_at_B) -- (SynB);
        \draw[self_data, draw=nips_blue] (RealB.south east) to node[midway, below] {Self $\mathcal{D}_2$} (SynB.south west);

        \begin{scope}[on background layer]
            \node[silo, fit=(RealB) (ModelB), label={[text=nips_blue, font=\bfseries\small, anchor=south west]north west:Client B (Class 2)}] {};
        \end{scope}

        \node[database, draw=nips_green, bottom color=nips_green!10] (RealC) at (0, -3.5) {Real\\$\mathcal{D}_3$};
        \node[factory, right=1.3cm of RealC] (FactC) {Factory\\C};
        \draw[flow] (RealC) -- node[math_label] {$x \sim p_3(x)$} (FactC);

        \node[decoder_mini, draw=nips_red, right=3.0cm of FactC, yshift=0.3cm] (DecA_at_C) {$G_{\theta_1}$};
        \node[decoder_mini, draw=nips_blue, below=0.1cm of DecA_at_C] (DecB_at_C) {$G_{\theta_2}$};
        \node[database, scale=0.85, right=0.5cm of DecA_at_C, yshift=-0.3cm, pattern=north east lines, pattern color=gray!40] (SynC) {$\text{Union } \hat{\mathcal{D}}_C$\\$Real \cup Syn$};
        \node[model_local, right=0.8cm of SynC] (ModelC) {Model\\$\textbf{w}_C$};
        \draw[flow] (SynC) -- (ModelC);
        
        \draw[->, dashed, gray] (DecA_at_C) -- (SynC);
        \draw[->, dashed, gray] (DecB_at_C) -- (SynC);
        \draw[self_data, draw=nips_green] (RealC.south east) to node[midway, below] {Self $\mathcal{D}_3$} (SynC.south west);

        \begin{scope}[on background layer]
            \node[silo, fit=(RealC) (ModelC), label={[text=nips_green, font=\bfseries\small, anchor=south west]north west:Client C (Class 3)}] (SiloC) {};
        \end{scope}

        
        \draw[bus_line, nips_red] (FactA.east) -- node[above, font=\tiny\bfseries, text=nips_red, yshift=0.5pt] {Broadcast $\theta_1$} ++(2.1,0) coordinate(Aout) -- (Aout |- DecA_at_C.west);
        
        \node[dot, nips_red] at (Aout |- DecA_at_B.west) {};
        \draw[->, nips_red, thick] (Aout |- DecA_at_B.west) -- (DecA_at_B.west);
        
        \node[dot, nips_red] at (Aout |- DecA_at_C.west) {};
        \draw[->, nips_red, thick] (Aout |- DecA_at_C.west) -- (DecA_at_C.west);

        \draw[bus_line, nips_blue] (FactB.east) -- node[above, font=\tiny\bfseries, text=nips_blue, pos=0.45pt, yshift=0.5pt] {Broadcast $\theta_2$} ++(2.25,0) coordinate(Bout) -- (Bout |- DecB_at_C.west); 
        \draw[bus_line, nips_blue] (Bout) -- (Bout |- DecB_at_A.west); 
        
        \node[dot, nips_blue, yshift=-0.8] at (Bout) {};
        
        \node[dot, nips_blue] at (Bout |- DecB_at_A.west) {};
        \draw[->, nips_blue, thick] (Bout |- DecB_at_A.west) -- (DecB_at_A.west);
        
        \node[dot, nips_blue] at (Bout |- DecB_at_C.west) {};
        \draw[->, nips_blue, thick] (Bout |- DecB_at_C.west) -- (DecB_at_C.west);

        \draw[bus_line, nips_green] (FactC.east) -- node[above, font=\tiny\bfseries, text=nips_green, pos=0.43pt, yshift=0.5pt] {Broadcast $\theta_3$} ++(2.4,0) coordinate(Cout) -- (Cout |- DecC_at_A.west);
        
        \node[dot, nips_green] at (Cout |- DecC_at_B.west) {};
        \draw[->, nips_green, thick] (Cout |- DecC_at_B.west) -- (DecC_at_B.west);
        
        \node[dot, nips_green] at (Cout |- DecC_at_A.west) {};
        \draw[->, nips_green, thick] (Cout |- DecC_at_A.west) -- (DecC_at_A.west);

        \node[rectangle, rounded corners=6pt, draw=nips_dark, line width=1.2pt, fill=white, 
              minimum height=2.5cm, minimum width=2.5cm, align=center,
              drop shadow={opacity=0.2, shadow xshift=3pt}] (POE) at (14.2, 0) {\textbf{Product of}\\\textbf{Experts}};
              
        \draw[flow] (ModelA.east) to[out=0, in=165] (POE.west);
        \draw[flow] (ModelB.east) -- (POE.west);
        \draw[flow] (ModelC.east) to[out=0, in=195] (POE.west);

        \node[below=0.6cm of POE, font=\small\bfseries] (Inf) {Inference};
        \draw[->, very thick, nips_dark] (POE.south) -- (Inf);
        
        
        \node[anchor=north west, align=left, font=\scriptsize\color{gray}, inner xsep=0pt] 
    at ($(SiloC.south west) + (0, -0.15)$) 
    {(i) Local Training\\(e.g. EDM2 \cite{Karras2024EDM2})};

        \node[anchor=north west, align=right, font=\scriptsize\color{gray}, inner xsep=0pt] 
            at ($(SiloC.south west -| DecA_at_C.west) + (0.30, -0.15)$) 
            {(ii) Broadcast \quad (iii) Classifier Training\\(e.g. ResNet-50 \cite{He2016ResNet})  };

        \node[below=0.2cm of Inf, font=\scriptsize\color{gray}, align=center] 
            {(iv) Inference:\\$\hat{P}(y|x) \propto \prod P_k(y|x)$};

    \end{tikzpicture}
    }
    \caption{\textbf{Decentralized FederatedFactory Protocol.} This architecture depicts the P2P topology where local data flows ($x \sim p_k(x)$) are augmented with synthetic samples from broadcasted Factories to train local experts, aggregated via PoE.}
    \label{fig:decentralized_workflow}
\end{figure}
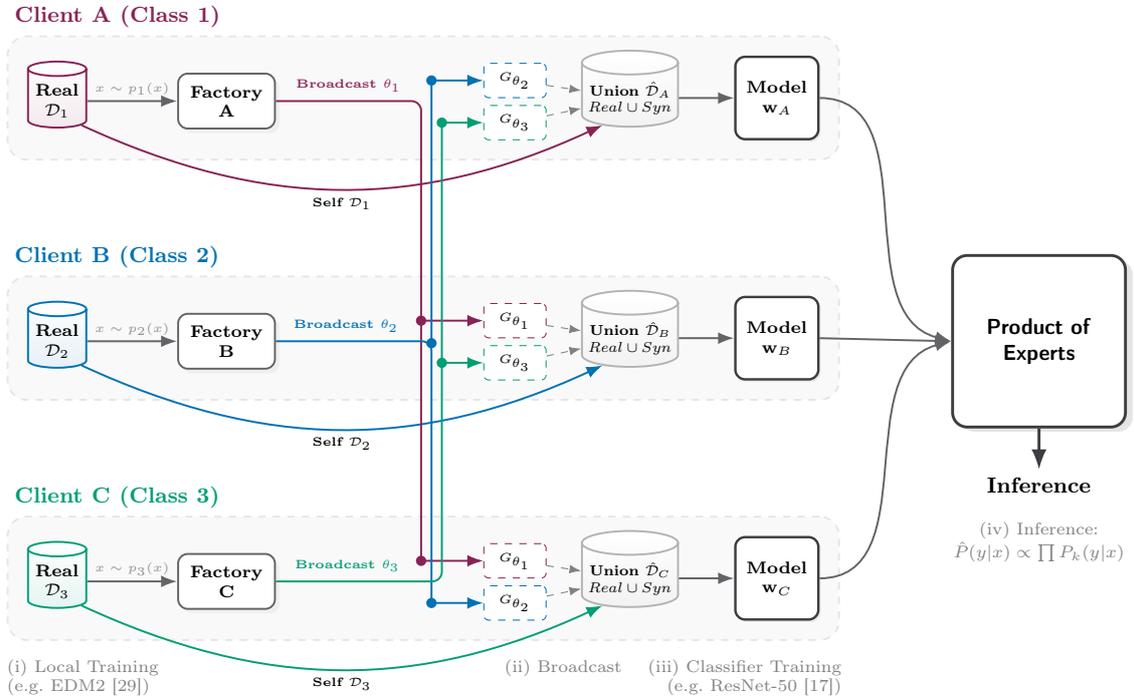


\subsection{Multi-Class Generalization}
\label{subsec:multi_class_generalization}

\textsc{FederatedFactory} natively extends to multi-class configurations. When clients possess data spanning multiple classes, each client $k \in \{1, \dots, K\}$ independently trains a class-specific Factory $G_{\boldsymbol{\theta}_k}(\mathbf{z}, c)$ for every class $c \in \mathcal{Y}_k$ present in its local dataset. To preserve the true global prior probability of the classes during centralized aggregation, the server requires the local sample counts $n_{c,k} = |\mathcal{D}_{c,k}|$. During global data synthesis, the server draws $m_{c,k}$ latent vectors for each specific Factory, strictly ensuring that the synthesized contribution is proportional to the local empirical density ($m_{c,k} \propto n_{c,k}$). 
To bypass the destructive non-convexity of parametric averaging, which frequently shifts mature weights into high-loss regions, the server instead performs \emph{data-space weighted aggregation}. For a target sample size $N_{\mathrm{target}}$ of class $c$, and clients $\mathcal{S}_c \subseteq \{1, \dots, K\}$ possessing this class, the server assigns a generation quota $Q_{c,k}$ to each local Factory $G_{\boldsymbol{\theta}_{c,k}}$:
$Q_{c,k} = N_{\mathrm{target}} \frac{n_{c,k}}{\sum_{j \in \mathcal{S}_c} n_{c,j}}$. 

This allocation ensures the synthetic global dataset $\hat{\mathcal{D}}_{\mathrm{syn}}$ mirrors its empirical sources. Large institutions establish the manifold's backbone, while smaller clinics inject stochastic diversity. This preserves rare morphological subtypes and scales the global distribution seamlessly, bypassing the interference of weight aggregation across imbalanced models.
The multi-class expansion organizes the aggregated models into a ``Visual Memory''~\cite{Geirhos2025Towards}, defining a \emph{Generative Matrix} $\boldsymbol{\Gamma} \in \mathcal{F}^{C \times K}$, with $C$ global classes, $K$ clients, and mapping space $\mathcal{F}$. Each entry $\boldsymbol{\Gamma}_{c,k}$ holds the localized Factory $G_{\boldsymbol{\theta}_{c,k}}$ (or $\emptyset$ if locally absent). Under a strict \emph{single-class silo} regime ($\alpha \to 0$), $\boldsymbol{\Gamma}$ collapses into a severely sparse diagonal, with each client contributing exactly one valid prior. Under general multi-class heterogeneity, it populates organically based on local label supports. 


\subsection{Modular Machine Unlearning and Exact Erasure}
\label{subsec:unlearning}

\begin{figure*}[t]
\centering
\resizebox{\textwidth}{!}{
\begin{tikzpicture}[>=stealth, font=\small, scale=0.8, every node/.style={transform shape}]
    
    \definecolor{basered}{RGB}{255, 235, 235}    
    \definecolor{darkred}{RGB}{136, 34, 85}    
    \definecolor{unlearnred}{RGB}{136, 34, 85}
    
    \newcommand{\genmatrix}[3]{
        \begin{scope}[shift={(#1,0)}]
            \node[font=\bfseries] at (1.1, 3.8) {#2};
            
            \foreach \x/\lbl in {0/k_1, 1.1/k_2, 2.2/k_3} {
                \node[font=\small\bfseries, text=gray] at (\x, 3.2) {$\lbl$};
            }
            
            \foreach \y/\lbl in {2.4/c_1, 1.3/c_2, 0.2/c_3} {
                \node[font=\small\bfseries, text=gray, anchor=east] at (-0.6, \y) {$\lbl$};
            }

            \foreach \x in {0, 1.1, 2.2} {
                \foreach \y in {0.2, 1.3, 2.4} {
                    
                    \def\isdeleted{0}

                    \ifnum #3=1 
                        \ifdim \x pt=1.1pt \def\isdeleted{1} \fi
                    \fi
                    \ifnum #3=2
                        \ifdim \y pt=1.3pt \def\isdeleted{1} \fi
                    \fi
                    \ifnum #3=3
                        \ifdim \x pt=1.1pt \ifdim \y pt=1.3pt \def\isdeleted{1} \fi \fi
                    \fi

                    \ifnum \isdeleted=1
                        \node[draw=unlearnred, fill=basered, thick, dashed, opacity=0.4, rounded corners=2pt, minimum width=0.95cm, minimum height=0.8cm] at (\x, \y) {};
                        \node[text=unlearnred, font=\tiny\bfseries, opacity=0.5] at (\x, \y) {$\emptyset$};
                    \else
                        \node[draw=gray!50, fill=white, rounded corners=2pt, minimum width=0.95cm, minimum height=0.8cm] at (\x, \y) {\tiny $G_{\boldsymbol{\theta}_{c,k}}$};
                    \fi
                }
            }
            
            \ifnum #3=1 
                \draw[unlearnred, line width=2pt, opacity=0.3] (1.1, 3.0) -- (1.1, -0.4);
            \fi
            \ifnum #3=2 
                \draw[unlearnred, line width=2pt, opacity=0.3] (-0.5, 1.3) -- (2.7, 1.3);
            \fi
            \ifnum #3=3 
                \draw[unlearnred, opacity=0.4] (1.1, 1.9) -- (1.1, 0.7);
                \draw[unlearnred, opacity=0.4] (0.5, 1.3) -- (1.7, 1.3);
            \fi
        \end{scope}
    }

    \genmatrix{0}{(1) Vertical}{1}
    \genmatrix{4.8}{(2) Horizontal}{2}
    \genmatrix{9.6}{(3) Targeted}{3}

    \node[anchor=north, align=center, font=\scriptsize, text=darkred] at (1.1, -0.6) {Client Removal\\$\boldsymbol{\Gamma}_{:, k} \leftarrow \emptyset$};
    \node[anchor=north, align=center, font=\scriptsize, text=darkred] at (5.9, -0.6) {Concept Erasure\\$\boldsymbol{\Gamma}_{c, :} \leftarrow \emptyset$};
    \node[anchor=north, align=center, font=\scriptsize, text=darkred] at (10.7, -0.6) {Specific Intersection\\$\boldsymbol{\Gamma}_{c, k} \leftarrow \emptyset$};

\end{tikzpicture}
}
\caption{\textbf{Modular Unlearning Modes in the Generative Matrix $\boldsymbol{\Gamma}$.} By structuring the global model as a disjoint union of class-client generators $G_{\boldsymbol{\theta}_{c,k}}$, \textsc{FederatedFactory} enables exact erasure without retraining the entire ensemble.}
\label{fig:modular_unlearning_modes}
\end{figure*}
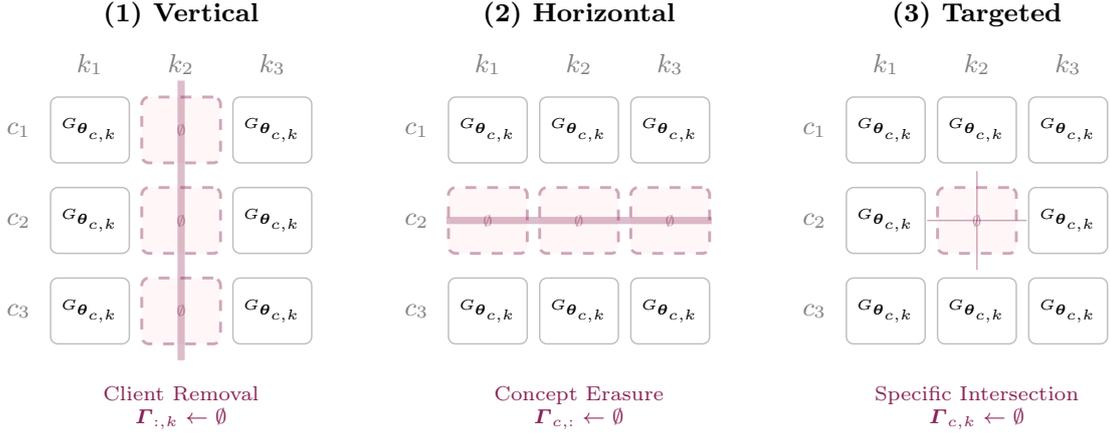


Decoupling the generative process into $\boldsymbol{\Gamma}$ structurally allows for modular unlearning~\cite{Bourtoule2021SISA}. Standard \textsc{FL} densely entangles representations across global weights, making localized data deletion intractable and typically requiring complete retraining to satisfy right-to-be-forgotten mandates~\cite{Golatkar2020EternalSunshine}. 

In \textsc{FederatedFactory}, the global representation is a discrete union of parameter-independent modules. In the single-class silo regime, unlearning trivially requires excising the target client's generative prior: $\boldsymbol{\Theta}_{\mathrm{new}} = \boldsymbol{\Theta} \setminus \{G_{\boldsymbol{\theta}_k}\}$. The server then discards the associated synthetic samples and retrains the centralized classifier. In decentralized \textsc{P2P} setups, this is mirrored locally: peers delete the revoked prior $G_{\boldsymbol{\theta}_k}$, flush related synthetic data, and retrain experts. 

Extending this logic through $\boldsymbol{\Gamma}$ (Figure~\ref{fig:modular_unlearning_modes}) enables three granular modes of exact data erasure. \textbf{Vertical Unlearning (Client Removal):} The server nullifies a column ($\boldsymbol{\Gamma}_{:, k} \leftarrow \emptyset$) to comprehensively forget client $k$. \textbf{Horizontal Unlearning (Concept Erasure):} The server executes a row-wise deletion ($\boldsymbol{\Gamma}_{c, :} \leftarrow \emptyset$) to remove an obsolete or restricted class $c$ consortium-wide. \textbf{Targeted Unlearning (Specific Intersection):} The server zeroes an exact coordinate ($\boldsymbol{\Gamma}_{c, k} \leftarrow \emptyset$) to remove a specific class subset from a specific client.

Following any deletion, the server clears the invalidated synthetic buffer $\hat{\mathcal{D}}_{\mathrm{syn}}$ and retrains the centralized classifier. Because the target manifold's underlying generative prior is eradicated, we guarantee exact data removal ($\hat{\mathcal{D}}_{\mathrm{syn}} \cap \hat{\mathcal{M}}_{c,k} = \emptyset$). This achieves true exact unlearning without approximations~\cite{Bourtoule2021SISA, Golatkar2020EternalSunshine, Yan2022ARCANE}.

\section{Experiments and Results}
\label{sec:experiments}

\paragraph{\textbf{Experimental Setup.}}
\label{subsec:exp_setup}
To validate \textsc{FederatedFactory}, we evaluate across three dataset regimes: \textbf{(1)} \textsc{CIFAR-10}~\cite{Krizhevsky2009CIFAR} as a standard high-variance baseline for foundational stability; \textbf{(2)} three \textsc{MedMNIST}~\cite{Yang2023MedMNIST} subsets (\textsc{BloodMNIST}, \textsc{RetinaMNIST}, \textsc{PathMNIST}) to test diverse morphological heterogeneity; and \textbf{(3)} \textsc{ISIC2019}~\cite{Terrail2022FLamby} as a high-resolution stress test for rare dermatoscopic classes.  Across all configurations, we employ an adaptive ResNet-50~\cite{He2016ResNet} backbone trained via an identical SGD optimizer (batch size 128, weight decay $1\times 10^{-4}$, initial learning rate 0.1 with cosine annealing \cite{Loshchilov2017SGDR}) and domain-specific augmentations over equivalent computational budgets (300 centralized/synthetic epochs vs. 200 rounds of 5 local epochs for federated baselines). 
Following preliminary grid-search tuning, we report the mean and standard deviation of Test Accuracy and macro-averaged One-vs-Rest (OvR) AUROC across five independent random seeds for all configurations. The centralized upper bound is explicitly trained on the theoretically inaccessible global dataset $\mathcal{D}_{\mathrm{union}}$.

\paragraph{\textbf{Results.}}
\label{subsec:performance_collapse}

Table~\ref{tab:master_performance} and Figure~\ref{fig:federatedfactory_merged_improvement} quantify the transition from moderate heterogeneity (Dirichlet $\alpha=0.1$) to the \textit{Single-Class Silo} extreme. Under moderate skew, \textsc{FedProx}~\cite{Li2020FedProx} and \textsc{SCAFFOLD}~\cite{Karimireddy2020SCAFFOLD} maintain convergence. However, in strictly disjoint silos, parameter aggregation degenerates into a random walk, causing unconditional baseline failure. \textsc{FedAvg} on \textsc{CIFAR} plunges from 89.76\% $\to$ 11.36\% accuracy, while on imbalanced sets like \textsc{ISIC2019}, global models collapse to majority-class prediction (e.g., \textsc{FedProx} drops from 62.83\% $\to$ 48.22\%). 
Conversely, \textsc{FederatedFactory} avoids this failure, as demonstrated on \textsc{RetinaMNIST} (49.30\% vs. 47.20\% Centralized Accuracy) and \textsc{ISIC2019} (90.57\% vs. 90.37\% Centralized AUROC).

\definecolor{colorFedAvg}{HTML}{F0E442}
\definecolor{colorFedDyn}{HTML}{E69F00}
\definecolor{colorFedProx}{HTML}{D55E00}
\definecolor{colorScaffold}{HTML}{882255}
\definecolor{colorFedFactC}{HTML}{56B4E9}
\definecolor{colorFedFactD}{HTML}{0072B2}
\definecolor{colorUpperBound}{HTML}{009E73}

\definecolor{colorFedAvg}{HTML}{F0E442}
\definecolor{colorFedDyn}{HTML}{E69F00}
\definecolor{colorFedProx}{HTML}{D55E00}
\definecolor{colorScaffold}{HTML}{882255}
\definecolor{colorFederatedFactoryCentralized}{HTML}{56B4E9}
\definecolor{colorFederatedFactoryDecentralized}{HTML}{0072B2}
\definecolor{colorUpperBound}{HTML}{009E73}

\begin{table}[t]
    \centering
    \setlength{\tabcolsep}{3pt}
    \caption{\textbf{Robustness to Extreme Statistical Heterogeneity.} Mean $\pm$ SD of Accuracy and AUROC (\%) under moderate ($\alpha=0.1$) and pathological ($\alpha \rightarrow 0$) label skew. Best overall performance (\textcolor{acadred}{Red}). \textbf{Bold} indicates the best FL method.}
    \label{tab:master_performance}
    
    \resizebox{\textwidth}{!}{%
    \begin{tabular}{l cccc cccc cccc}
        \toprule
        \multirow{3}{*}{\textbf{Method}} & \multicolumn{4}{c}{\textbf{CIFAR}} & \multicolumn{4}{c}{\textbf{BloodMNIST}} & \multicolumn{4}{c}{\textbf{PathMNIST}} \\
        \cmidrule(lr){2-5} \cmidrule(lr){6-9} \cmidrule(lr){10-13}
        & \multicolumn{2}{c}{\textbf{Dirichlet}} & \multicolumn{2}{c}{\textbf{Silos}} 
        & \multicolumn{2}{c}{\textbf{Dirichlet}} & \multicolumn{2}{c}{\textbf{Silos}} 
        & \multicolumn{2}{c}{\textbf{Dirichlet}} & \multicolumn{2}{c}{\textbf{Silos}} \\
        \cmidrule(lr){2-3} \cmidrule(lr){4-5} \cmidrule(lr){6-7} \cmidrule(lr){8-9} \cmidrule(lr){10-11} \cmidrule(lr){12-13}
        & Acc & AUC & Acc & AUC & Acc & AUC & Acc & AUC & Acc & AUC & Acc & AUC \\
        \midrule
        
        \rowcolor{colorFedAvg!20} 
        \textbf{FedAvg \cite{McMahan2017FedAvg}} 
        & \textbf{89.76 $\pm$ 1.12} & \textbf{99.22 $\pm$ 0.19} & 11.36 $\pm$ 1.28 & 50.91 $\pm$ 1.79 
        & 83.46 $\pm$ 5.18 & 97.57 $\pm$ 1.11 & 21.88 $\pm$ 5.39 & 55.23 $\pm$ 7.17 
        & 73.79 $\pm$ 16.86 & 90.11 $\pm$ 16.05 & 18.05 $\pm$ 0.80 & 50.12 $\pm$ 1.76 \\
        
        \rowcolor{colorFedDyn!20} 
        \textbf{FedDyn \cite{Acar2021FedDyn}} 
        & 61.82 $\pm$ 11.47 & 94.87 $\pm$ 2.65 & 10.12 $\pm$ 0.24 & 51.48 $\pm$ 0.83 
        & 64.93 $\pm$ 8.63 & 92.02 $\pm$ 2.21 & 19.47 $\pm$ 0.00 & 51.44 $\pm$ 2.87 
        & 60.99 $\pm$ 8.66 & 92.71 $\pm$ 4.49 & 17.76 $\pm$ 0.80 & 47.57 $\pm$ 3.22 \\
        
        \rowcolor{colorFedProx!20} 
        \textbf{FedProx \cite{Li2020FedProx}} 
        & 89.67 $\pm$ 1.15 & 99.21 $\pm$ 0.19 & 18.08 $\pm$ 1.18 & 67.06 $\pm$ 1.02 
        & \textbf{84.11 $\pm$ 5.38} & \textbf{97.61 $\pm$ 1.00} & 20.18 $\pm$ 1.57 & 53.69 $\pm$ 3.32 
        & \textbf{77.09 $\pm$ 11.89} & \textbf{96.01 $\pm$ 4.44} & 18.93 $\pm$ 0.57 & 52.97 $\pm$ 2.67 \\
        
        \rowcolor{colorScaffold!20} 
        \textbf{Scaffold \cite{Karimireddy2020SCAFFOLD}} 
        & 79.26 $\pm$ 5.10 & 97.73 $\pm$ 0.86 & 12.99 $\pm$ 0.75 & 54.83 $\pm$ 1.08 
        & 82.31 $\pm$ 2.97 & 97.51 $\pm$ 0.75 & 22.60 $\pm$ 1.96 & 67.84 $\pm$ 2.98 
        & 71.88 $\pm$ 9.70 & 95.61 $\pm$ 1.88 & 31.15 $\pm$ 2.10 & 64.87 $\pm$ 2.79 \\
        
        \midrule
        
        \rowcolor{colorFederatedFactoryCentralized!20} 
        \textsc{FedFact} (Cent.) 
        & -- & -- & 84.30 $\pm$ 0.66 & 98.24 $\pm$ 0.13 
        & -- & -- & \textbf{91.17 $\pm$ 0.26} & \textbf{99.04 $\pm$ 0.06} 
        & -- & -- & \textbf{67.94 $\pm$ 3.84} & \textbf{93.60 $\pm$ 1.07} \\
        
        \rowcolor{colorFederatedFactoryDecentralized!20} 
        \textsc{FedFact} (P2P) 
        & -- & -- & \textbf{90.57 $\pm$ 0.09} & \textbf{99.14 $\pm$ 0.02} 
        & -- & -- & 86.38 $\pm$ 0.31 & 98.36 $\pm$ 0.07 
        & -- & -- & 67.03 $\pm$ 2.28 & 91.48 $\pm$ 0.64 \\
        
        \midrule
        
        \rowcolor{colorUpperBound!15}
        Centralized Upper Bound
        & \textit{\textcolor{acadred}{94.69 $\pm$ 0.33}} & \textit{\textcolor{acadred}{99.75 $\pm$ 0.03}} & \textit{\textcolor{acadred}{94.69 $\pm$ 0.33}} & \textit{\textcolor{acadred}{99.75 $\pm$ 0.03}} 
        & \textit{\textcolor{acadred}{91.23 $\pm$ 0.87}} & \textit{\textcolor{acadred}{99.18 $\pm$ 0.13}} & \textit{\textcolor{acadred}{91.23 $\pm$ 0.87}} & \textit{\textcolor{acadred}{99.18 $\pm$ 0.13}} 
        & \textit{\textcolor{acadred}{84.82 $\pm$ 0.81}} & \textit{\textcolor{acadred}{96.96 $\pm$ 0.40}} & \textit{\textcolor{acadred}{84.82 $\pm$ 0.81}} & \textit{\textcolor{acadred}{96.96 $\pm$ 0.40}} \\
        \bottomrule
    \end{tabular}%
    }
    
    \vspace{0.4cm}
    
    \resizebox{0.80\textwidth}{!}{%
    \begin{tabular}{l cccc cccc}
        \toprule
        \multirow{3}{*}{\textbf{Method}} & \multicolumn{4}{c}{\textbf{RetinaMNIST}} & \multicolumn{4}{c}{\textbf{ISIC2019}} \\
        \cmidrule(lr){2-5} \cmidrule(lr){6-9}
        & \multicolumn{2}{c}{\textbf{Dirichlet}} & \multicolumn{2}{c}{\textbf{Silos}} 
        & \multicolumn{2}{c}{\textbf{Dirichlet}} & \multicolumn{2}{c}{\textbf{Silos}} \\
        \cmidrule(lr){2-3} \cmidrule(lr){4-5} \cmidrule(lr){6-7} \cmidrule(lr){8-9}
        & Acc & AUC & Acc & AUC & Acc & AUC & Acc & AUC \\
        \midrule
        
        \rowcolor{colorFedAvg!20} 
        \textbf{FedAvg \cite{McMahan2017FedAvg}} 
        & 45.20 $\pm$ 3.80 & \textbf{55.02 $\pm$ 7.96} & 43.50 $\pm$ 0.00 & 48.70 $\pm$ 1.90 
        & 60.35 $\pm$ 6.94 & 76.73 $\pm$ 12.33 & 48.22 $\pm$ 0.00 & 47.31 $\pm$ 1.78 \\
        
        \rowcolor{colorFedDyn!20} 
        \textbf{FedDyn \cite{Acar2021FedDyn}} 
        & 45.25 $\pm$ 3.91 & 54.02 $\pm$ 8.87 & 43.50 $\pm$ 0.00 & 48.70 $\pm$ 1.92 
        & 53.63 $\pm$ 5.48 & 68.48 $\pm$ 7.81 & 48.22 $\pm$ 0.00 & 43.28 $\pm$ 0.99 \\
        
        \rowcolor{colorFedProx!20} 
        \textbf{FedProx \cite{Li2020FedProx}} 
        & \textbf{45.45 $\pm$ 4.36} & 54.52 $\pm$ 8.70 & 43.50 $\pm$ 0.00 & 48.74 $\pm$ 1.91 
        & \textbf{62.83 $\pm$ 2.41} & \textbf{83.21 $\pm$ 1.60} & 48.22 $\pm$ 0.00 & 45.49 $\pm$ 0.78 \\
        
        \rowcolor{colorScaffold!20} 
        \textbf{Scaffold \cite{Karimireddy2020SCAFFOLD}} 
        & 43.50 $\pm$ 0.00 & 52.52 $\pm$ 6.15 & 43.50 $\pm$ 0.00 & 48.70 $\pm$ 1.90 
        & 50.03 $\pm$ 1.75 & 64.37 $\pm$ 6.19 & 48.22 $\pm$ 0.00 & 44.33 $\pm$ 1.02 \\
        
        \midrule
        
        \rowcolor{colorFederatedFactoryCentralized!20} 
        \textsc{FedFact} (Cent.) 
        & -- & -- & 46.75 $\pm$ 1.29 & 70.29 $\pm$ 1.06 
        & -- & -- & 62.08 $\pm$ 0.59 & 84.98 $\pm$ 0.57 \\
        
        \rowcolor{colorFederatedFactoryDecentralized!20} 
        \textsc{FedFact} (P2P) 
        & -- & -- & \textbf{\textcolor{acadred}{49.30 $\pm$ 1.25}} & \textbf{\textcolor{acadred}{71.79 $\pm$ 0.64}} 
        & -- & -- & \textbf{69.94 $\pm$ 0.46} & \textbf{\textcolor{acadred}{90.57 $\pm$ 0.08}} \\
        
        \midrule
        
        \rowcolor{colorUpperBound!15}
        Centralized Upper Bound
        & \textit{47.20 $\pm$ 0.78} & \textit{69.69 $\pm$ 1.00} & \textit{47.20 $\pm$ 0.78} & \textit{69.69 $\pm$ 1.00} 
        & \textit{\textcolor{acadred}{70.38 $\pm$ 0.69}} & \textit{90.37 $\pm$ 0.74} & \textit{\textcolor{acadred}{70.38 $\pm$ 0.69}} & \textit{90.37 $\pm$ 0.74} \\
        \bottomrule
    \end{tabular}%
    }
\end{table}

\begin{figure}[t]
    \centering
    \includegraphics[width=\textwidth]{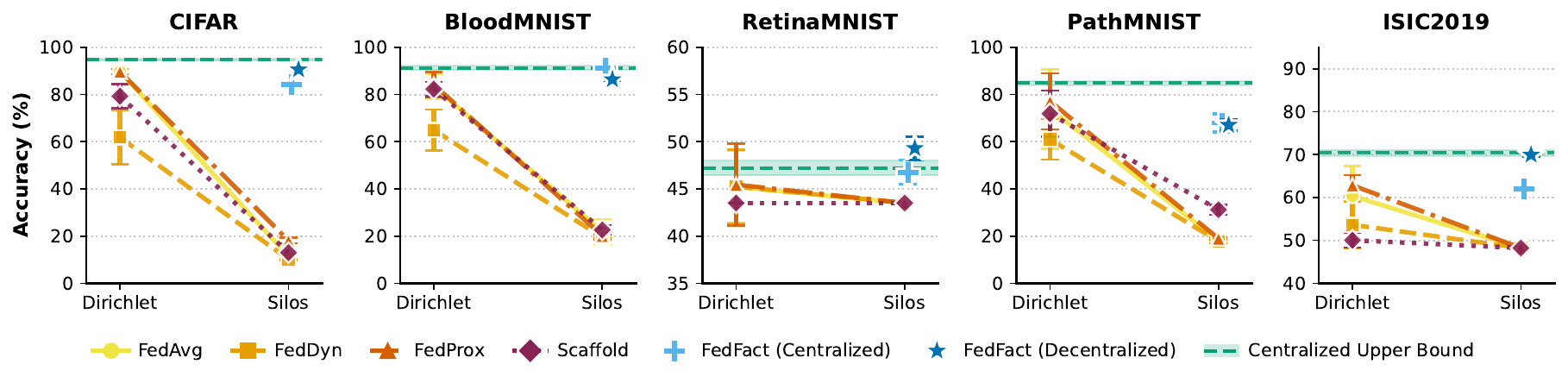}

    
    \caption{\textbf{Results in Pathological Heterogeneity.} While standard baselines (\legendbox{colorFedAvg!20}{\textbf{FedAvg $\circ$}}, \legendbox{colorFedDyn!20}{\textbf{FedDyn $\square$}}, \legendbox{colorFedProx!20}{\textbf{FedProx $\triangle$}}, \legendbox{colorScaffold!20}{\textbf{Scaffold $\diamond$}}) collapse as we move from moderate skew ($\alpha=0.1$) to extreme silos ($\alpha \rightarrow 0$), \textsc{FederatedFactory} matches the \legendbox{colorUpperBound!20}{\textbf{Upper Bound (---)}} in both \legendbox{colorFedFactC!20}{\textbf{Centralized\ ($+$)}}, \legendbox{colorFedFactD!20}{\textbf{Decentralized\ ($\star$)}} configurations.}
    \label{fig:federatedfactory_merged_improvement}
\end{figure}

\begin{takeaway}
\textsc{FederatedFactory} fixes the \textit{Single-Class Silo} collapse, and it also manages to match the centralized data-pooled upper-bound baseline, under more difficult constraints.
\end{takeaway}

\paragraph{\textbf{Communication-Computation Anaysis.}}
\label{subsec:flops_vs_mbs}
Table~\ref{tab:combined_evaluation} highlights \textsc{FederatedFactory}'s shift from a bandwidth-bound to a compute-bound regime. Traditional FL baselines minimize edge compute but incur communication penalties from iterative convergence. Bidirectional parameter transmissions across 200 rounds create severe bottlenecks, demanding hundreds of gigabytes (e.g., 358,899.6 MB on \textsc{CIFAR-10}), a vulnerability \textsc{SCAFFOLD} strictly doubles for control variates.
Conversely, \textsc{FederatedFactory} employs a OSFL protocol, collapsing network exposure and achieving a 99.4\% communication reduction (1,934.1 MB on \textsc{CIFAR-10}). This zero-dependency design incurs a ``generative tax'' in local FLOPs, scaling computational load from $\sim 10^{17}$ in baselines to $\sim 10^{18}$ (up to $\sim 10^{20}$ for \textsc{ISIC2019}). Ultimately, trading computational overhead for reduced communication costs is justified in cross-silo clinical settings, where model trustworthiness takes precedence over computational expense.

\begin{table}[t]
    \centering
    \setlength{\tabcolsep}{3pt}
    \caption{\textbf{Resource Trade-off Analysis.} Total computational overhead (FLOPs) vs. communication volume (MBs) across the five benchmarks.}
    \label{tab:combined_evaluation}
    \resizebox{\textwidth}{!}{%
    \begin{tabular}{l cccccccccc}
        \toprule
        \multirow{2}{*}{\textbf{Method}} 
        & \multicolumn{2}{c}{\textbf{CIFAR}} 
        & \multicolumn{2}{c}{\textbf{BloodMNIST}} 
        & \multicolumn{2}{c}{\textbf{PathMNIST}} 
        & \multicolumn{2}{c}{\textbf{RetinaMNIST}} 
        & \multicolumn{2}{c}{\textbf{ISIC2019}} \\
        \cmidrule(lr){2-3} \cmidrule(lr){4-5} \cmidrule(lr){6-7} \cmidrule(lr){8-9} \cmidrule(lr){10-11}
        & \textbf{FLOPs} & \textbf{MBs} & \textbf{FLOPs} & \textbf{MBs} & \textbf{FLOPs} & \textbf{MBs} & \textbf{FLOPs} & \textbf{MBs} & \textbf{FLOPs} & \textbf{MBs} \\
        \midrule
        
        \rowcolor{colorFedAvg!20} 
        \textbf{FedAvg \cite{McMahan2017FedAvg}} 
        & $\mathbf{1.95 \times 10^{17}}$ & 358,899.6 
        & $\mathbf{3.57 \times 10^{16}}$ & 287,069.6 
        & $\mathbf{2.69 \times 10^{17}}$ & 322,981.5 
        & $\mathbf{3.22 \times 10^{15}}$ & 179,371.6 
        & $\mathbf{3.11 \times 10^{17}}$ & 287,163.4 \\
        
        \rowcolor{colorFedDyn!20} 
        \textbf{FedDyn \cite{Acar2021FedDyn}} 
        & $\mathbf{1.95 \times 10^{17}}$ & 358,899.6 
        & $\mathbf{3.57 \times 10^{16}}$ & 287,069.6 
        & $\mathbf{2.69 \times 10^{17}}$ & 322,981.5 
        & $\mathbf{3.22 \times 10^{15}}$ & 179,371.6 
        & $\mathbf{3.11 \times 10^{17}}$ & 287,163.4 \\
        
        \rowcolor{colorFedProx!20} 
        \textbf{FedProx \cite{Li2020FedProx}} 
        & $\mathbf{1.95 \times 10^{17}}$ & 358,899.6 
        & $\mathbf{3.57 \times 10^{16}}$ & 287,069.6 
        & $\mathbf{2.69 \times 10^{17}}$ & 322,981.5 
        & $\mathbf{3.22 \times 10^{15}}$ & 179,371.6 
        & $\mathbf{3.11 \times 10^{17}}$ & 287,163.4 \\
        
        \rowcolor{colorScaffold!20} 
        \textbf{Scaffold \cite{Karimireddy2020SCAFFOLD}} 
        & $\mathbf{1.95 \times 10^{17}}$ & 717,799.1 
        & $\mathbf{3.57 \times 10^{16}}$ & 574,139.3 
        & $\mathbf{2.69 \times 10^{17}}$ & 645,962.9 
        & $\mathbf{3.22 \times 10^{15}}$ & 358,743.2 
        & $\mathbf{3.11 \times 10^{17}}$ & 574,326.8 \\
        
        \midrule
        
        \rowcolor{colorFederatedFactoryCentralized!20} 
        \textsc{FedFact} (Server) 
        & $7.30 \times 10^{18}$ & \textbf{1,934.1} 
        & $4.47 \times 10^{18}$ & \textbf{1,547.3} 
        & $5.03 \times 10^{18}$ & \textbf{1,740.7} 
        & $2.79 \times 10^{18}$ & \textbf{967.0} 
        & $2.52 \times 10^{20}$ & \textbf{1,435.3} \\
        
        \rowcolor{colorFederatedFactoryDecentralized!20} 
        \textsc{FedFact} (Local) 
        & $8.35 \times 10^{18}$ & 19,340.6 
        & $4.97 \times 10^{18}$ & 12,378.0 
        & $5.67 \times 10^{18}$ & 15,665.9 
        & $2.97 \times 10^{18}$ & 4,835.2 
        & $2.54 \times 10^{20}$ & 11,482.0 \\
        \bottomrule
    \end{tabular}%
    }
\end{table}

\begin{figure}[t]
    \centering
    \includegraphics[width=\textwidth]{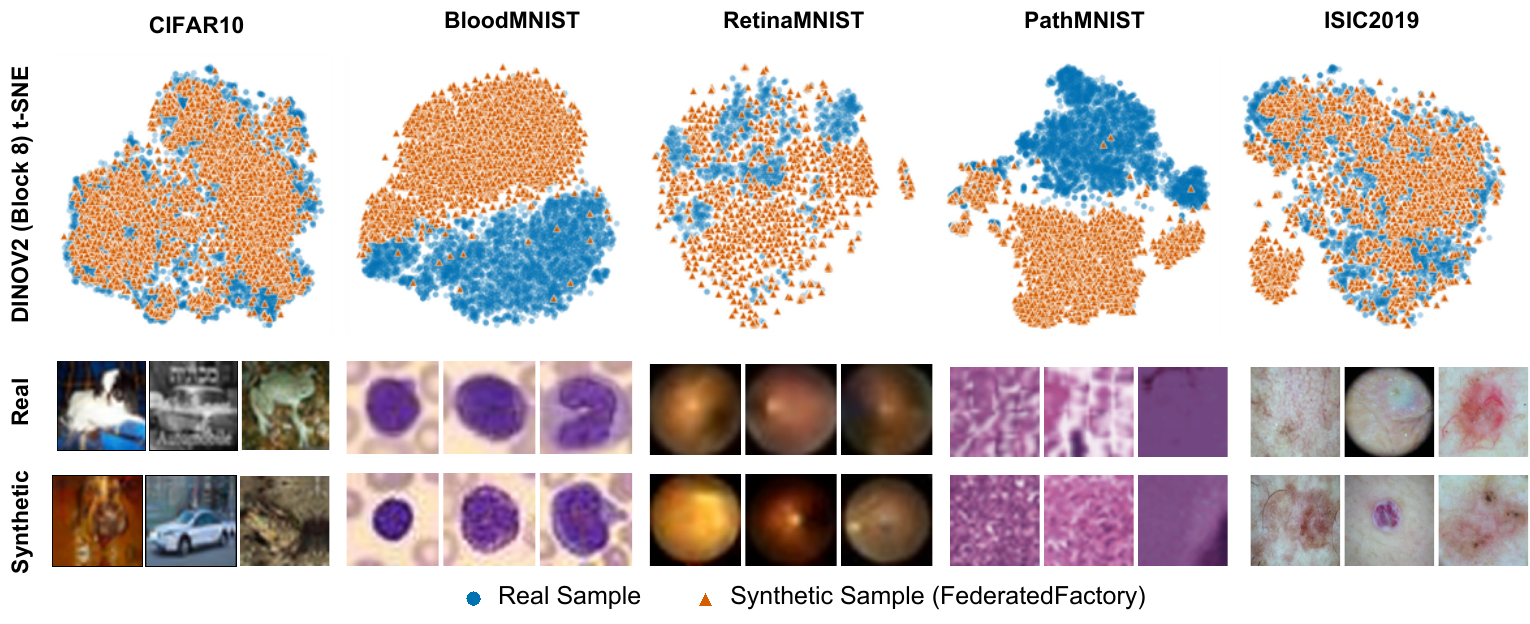}
    \caption{\textbf{Real vs. Synthetic Data in Representation Space.} \textbf{Top:} DINOv2 (Block 8) feature space~\cite{Oquab2023DINOv2}. FederatedFactory ($\blacktriangle$) represents the true target manifolds data (\textbullet), preserving good performance. \textbf{Bottom:} Visual comparisons across datasets. Columns pair real and generated samples of the same class, demonstrating accurate morphological preservation without memorization.}
    \label{fig:tsne_embeddings}
\end{figure}

\paragraph{\textbf{Qualitative Analysis and Manifold Alignment.}}
\label{subsec:qualitative_analysis}

To verify that Factories correctly approximate the localized manifolds without memorization~~\cite{Bonnaire2025Diffusion}, we project the generated distribution $\hat{\mathcal{D}}_{\mathrm{syn}}$ and the true target distribution $\mathcal{D}_{\mathrm{union}}$ into a joint feature space, presented alongside qualitative samples in Figure~\ref{fig:tsne_embeddings}. Visual comparisons (Fig.~\ref{fig:tsne_embeddings}, Top) demonstrate \textsc{FederatedFactory} successfully synthesizes complex morphologies (e.g., \textsc{ISIC2019} textures, \textsc{BloodMNIST} cells) \textit{ex nihilo}. To validate this structural fidelity, we project intermediate DINOv2 representations (Block 8) \cite{Oquab2023DINOv2} via t-SNE (Fig.~\ref{fig:tsne_embeddings}, Bottom), a feature space proven effective for isolating generative artifacts~\cite{Interno2026AiGenerated}. The synthetic priors natively span the true empirical distribution without discrete coordinate overlaps. While not a formal privacy guarantee, this empirically confirms continuous diversity and functional manifold mapping over trivial memorization.

\section{Conclusion}
As federated learning (FL) scales \cite{Kairouz2021Advances}, resolving extreme statistical heterogeneity is critical \cite{Hsu2019Dirichlet, Zhao2018NonIID}. To bypass the collapse of standard parameter aggregation under single-class silos \cite{McMahan2017FedAvg}, \textsc{FederatedFactory} transfers localized generative priors rather than discriminative gradients. By using independent generative modules to synthesize class-balanced global datasets \textit{ex nihilo}~\cite{Ho2020DDPM, Karras2022EDM}, our zero-dependency framework eliminates gradient conflict and external prior bias \cite{Huix2024Are, He2026Ai}. This consistently recovers centralized upper-bound performance where standard methods fail. For instance, under strictly disjoint silos, it lifts CIFAR-10 accuracy from a collapsed $11.36\%$ (FedAvg) to $90.57\%$ (matching the centralized bound), and improves ISIC2019 AUROC from $47.31\%$ to $90.57\%$ \cite{Terrail2022FLamby}. Remarkably, it achieves this while slashing communication overhead by $99.4\%$ (from $358,899.6$ MB to just $1,934.1$ MB on CIFAR-10) \cite{McMahan2017FedAvg}.

While traditional FL averages local trajectories, our results prove that transferring the underlying data manifold approximation is fundamentally more robust for disjoint label supports. This paradigm shift encourages transmitting localized generative models over fragile discriminative boundaries. Furthermore, structuring the global model as a discrete union of these generative modules inherently facilitates exact modular unlearning \cite{Bourtoule2021SISA, Yan2022ARCANE}. Excising a client's specific generative parameters guarantees exact data erasure, effortlessly complying with right-to-be-forgotten mandates without retraining the global ensemble \cite{Golatkar2020EternalSunshine}.

\textbf{Limitations.} Trading communication for compute bottlenecks in cross-silo settings introduces a substantial ``generative tax.'' Hardware profiling shows local computational loads scale by roughly an order of magnitude (e.g., $1.95 \times 10^{17}$ to $8.35 \times 10^{18}$ FLOPs on CIFAR-10). Furthermore, theoretical convergence assumes local diffusion models converge without mere memorization \cite{Bonnaire2025Diffusion}. While feature space projections empirically confirm continuous diversity, the framework lacks formal privacy guarantees (e.g., Differential Privacy)~\cite{0400000042} against advanced data extraction~\cite{3620531} or membership inference attacks on transmitted priors~\cite{hu2023membershipinferencediffusionmodels}.

\textbf{Broader Impacts.} Standard FL fragility endangers multi-institutional collaborations, especially in clinical imaging where data sovereignty is fundamental \cite{Shelleretal2020, Riekeetal2020}. \textsc{FederatedFactory} provides a robust, zero-dependency alternative. Moreover, organizing the global model as a discrete Generative Matrix inherently supports exact modular unlearning \cite{Bourtoule2021SISA, Yan2022ARCANE}, ensuring compliance with stringent data privacy regulations like the Right to be Forgotten and strictly protecting localized data rights \cite{Golatkar2020EternalSunshine}.

\bibliographystyle{splncs04}
\bibliography{main}
\newpage

\appendix

\section*{Appendix Contents}
\vspace{0.5em}
\noindent\rule{\linewidth}{0.8pt} 
\vspace{1.2em}

\newcommand{\tocsection}[2]{%
    \noindent\hyperref[#1]{\textbf{\ref*{#1}}\quad#2}\hfill\pageref{#1}\par\vspace{3pt}}
\newcommand{\tocsubsection}[2]{%
    \noindent\hspace{2em}\hyperref[#1]{\ref*{#1}\quad#2}\ \textcolor{black!50}{\dotfill}\ \pageref{#1}\par}

\tocsection{app:pseudocode}{FederatedFactory Pseudocode}
\tocsection{app:qualitative}{Qualitative Examples of Diffusion Generated Images}
\tocsection{app:ecdf}{Empirical Cumulative Distribution Function and t-SNE Analysis}
\tocsection{app:performance}{Additional FederatedFactory Performance Analysis}
\tocsection{app:cost}{Additional FederatedFactory Cost Analysis}
\tocsection{app:image_generation_metrics}{Image Generation Metrics (FID, KID)}
\tocsection{app:computational_environments}{Computational Environments}
\tocsection{app:datasets_and_sources}{Datasets and Sources}

\vspace{1em}
\noindent\rule{\linewidth}{0.4pt} 
\vspace{2em}

\section{FederatedFactory Pseudocode}
\label{app:pseudocode}

For reproducibility, we formalize the complete \textsc{FederatedFactory} operational protocol in Algorithm~\ref{alg:fedfact_unified}. The formulation unites both the Centralized and Decentralized architectures into a single execution graph, controlled by the structural indicator $\mathcal{T}$. 


\begin{algorithm}[htpb]
\caption{\textsc{FederatedFactory}: Unified Global Optimization Protocol}
\label{alg:fedfact_unified}
\begin{algorithmic}[1]
\Require Network of $K$ clients, private local datasets $\mathcal{D}_k$, corresponding local label supports $\mathcal{Y}_k$.
\Require Architecture $\mathcal{T} \in \{\textsc{Centralized}, \textsc{Decentralized}\}$.
\Require Generative epochs $E_{\mathrm{gen}}$, Discriminative epochs $E_{\mathrm{disc}}$, Generation quotas $Q_k$.
\Require Untrained generative architectures $\{G_{\boldsymbol{\theta}_k}\}_{k=1}^K$.

\vspace{0.1cm}
\Statex \textbf{Phase I: Local Generative Prior Optimization (Asynchronous)}
\For{each client $k \in \{1, \dots, K\}$ \textbf{in parallel}}
    \For{$e = 1$ \textbf{to} $E_{\mathrm{gen}}$}
        \State Sample local empirical batch $\mathbf{x} \sim \mathcal{D}_k$
        \State Compute diffusion objective: $\mathcal{L}_{\mathrm{ELBO}}(\boldsymbol{\theta}_k; \mathbf{x})$ \Comment{Minimizes $\mathrm{KL}(p_k \parallel p_{\boldsymbol{\theta}_k})$ bounded by $\epsilon_k$}
        \State Update local generative prior: $\boldsymbol{\theta}_k \leftarrow \boldsymbol{\theta}_k - \eta \nabla_{\boldsymbol{\theta}_k} \mathcal{L}_{\mathrm{ELBO}}$
    \EndFor
\EndFor

\vspace{0.1cm}
\Statex \textbf{Phase II: Conditional Architectural Execution \& Discriminative Training}
\If{$\mathcal{T} == \textsc{Centralized}$} \Comment{Trusted Aggregator Available}
    \State Clients transmit optimized parameters $\boldsymbol{\theta}_k$ to the Server \Comment{Satisfies $\mathcal{C}_{\mathrm{rounds}} = 1$}
    \State Server initializes empty global synthetic dataset: $\hat{\mathcal{D}}_{\mathrm{syn}} \leftarrow \emptyset$
    \For{each received model $G_{\boldsymbol{\theta}_k} \in \boldsymbol{\Theta}$}
        \For{$i = 1$ \textbf{to} $Q_k$}
            \State Sample latent noise vector $\mathbf{z} \sim \mathcal{N}(\mathbf{0}, \mathbf{I})$
            \State Synthesize counterfactual mapping $\hat{\mathbf{x}} \leftarrow G_{\boldsymbol{\theta}_k}(\mathbf{z})$
            \State $\hat{\mathcal{D}}_{\mathrm{syn}} \leftarrow \hat{\mathcal{D}}_{\mathrm{syn}} \cup \{(\hat{\mathbf{x}}, y_k)\}$ \Comment{\textit{Ex nihilo} synthesis without external \textsc{FM}}
        \EndFor
    \EndFor
    \State Server initializes centralized classifier $\mathbf{w}$
    \State Optimize $\mathbf{w}$ on $\hat{\mathcal{D}}_{\mathrm{syn}}$ for $E_{\mathrm{disc}}$ epochs using standard empirical risk minimization
\ElsIf{$\mathcal{T} == \textsc{Decentralized}$} \Comment{Trustless Decentralized Mesh}
    \For{each client $k \in \{1, \dots, K\}$ \textbf{in parallel}}
        \State Broadcast $\boldsymbol{\theta}_k$ to all valid peers $j \in \{1, \dots, K\} \setminus \{k\}$
        \State Receive complement priors $\boldsymbol{\Theta}_{\setminus k} = \{ \boldsymbol{\theta}_j \}_{j \neq k}$
        \State Initialize local hybrid dataset: $\mathcal{D}_k^{\mathrm{mix}} \leftarrow \mathcal{D}_k$
        \For{each received complement model $G_{\boldsymbol{\theta}_j} \in \boldsymbol{\Theta}_{\setminus k}$}
            \State Generate $Q_j$ samples from $G_{\boldsymbol{\theta}_j}(\mathbf{z})$ and append mappings to $\mathcal{D}_k^{\mathrm{mix}}$
        \EndFor
        \State Initialize local expert classifier $\mathbf{w}_k$
        \State Optimize $\mathbf{w}_k$ exclusively on $\mathcal{D}_k^{\mathrm{mix}}$ for $E_{\mathrm{disc}}$ epochs
    \EndFor
\EndIf

\vspace{0.1cm}
\Statex \textbf{Phase III: Distributed Global Inference}
\Procedure{Inference}{Target sample $\mathbf{x}_{\mathrm{target}}$}
    \If{$\mathcal{T} == \textsc{Centralized}$}
        \State \Return $p(y \mid \mathbf{x}_{\mathrm{target}} ; \mathbf{w})$
    \ElsIf{$\mathcal{T} == \textsc{Decentralized}$}
        \State Gather local unnormalized probabilities: $\mathbf{p}_k = f_{\mathbf{w}_k}(\mathbf{x}_{\mathrm{target}}), \quad \forall k \in \{1, \dots, K\}$
        \State Compute joint consensus: $p_{\mathrm{joint}}(y) = \prod_{k=1}^K p_k(y \mid \mathbf{x}_{\mathrm{target}})$
        \State Compute partition function: $Z = \sum_{y'} p_{\mathrm{joint}}(y')$
        \State \Return Renormalized Product of Experts: $p_{\mathrm{PoE}}(y \mid \mathbf{x}_{\mathrm{target}}) = \frac{1}{Z} p_{\mathrm{joint}}(y)$
    \EndIf
\EndProcedure
\end{algorithmic}
\end{algorithm}

\section{Qualitative Examples of Diffusion Generated Images}
\label{app:qualitative}

To visually validate the fidelity and diversity of the \textsc{FederatedFactory} synthesis, we perform a nearest-neighbor analysis across all evaluated domains. 

For each class, we randomly sample latent vectors $\mathbf{z} \sim \mathcal{N}(\mathbf{0}, \mathbf{I})$ to generate synthetic images via the localized EDM2 \cite{Karras2024EDM2} Factory. We then compute the pixel-wise Euclidean (L2) distance across the local training manifold to find each generated sample's closest real counterpart.

Figures \ref{fig:comparison_cifar}, \ref{fig:comparison_bloodmnist}, \ref{fig:comparison_pathmnist}, \ref{fig:comparison_retinamnist}, and \ref{fig:comparison_isic} present these comparisons. The top rows display the synthetic samples, while the bottom rows show their nearest real neighbors with the embedded L2 distances ($d$). Across both low-resolution standard benchmarks (\textsc{CIFAR-10}) and high-resolution medical modalities (\textsc{ISIC2019}), the generative priors successfully capture the underlying semantic morphologies and textures. Importantly, structural differences in poses, backgrounds, and boundaries between the synthetic images and their real matches confirm that the model synthesizes diverse data rather than memorizing the training set.

\section{Empirical Cumulative Distribution Function and t-SNE Analysis}
\label{app:ecdf}

We evaluate \textsc{FederatedFactory}'s manifold mapping by projecting generated ($\hat{\mathcal{D}}_{\mathrm{syn}}$) and local training samples ($\mathcal{D}_{\mathrm{union}}$) into 2D via t-SNE. As \cref{fig:tsne} shows, synthetic data consistently populates the same macroscopic regions as real data across all datasets. Crucially, despite this global semantic alignment, there is no point-to-point overlap. We further formalize the evaluation of inter-class fidelity and intra-class diversity through Empirical Cumulative Distribution Function (ECDF) curves and nearest-neighbor distance histograms (\cref{fig:generative_metrics}). 

\begin{figure}[t]
    \centering
    \includegraphics[width=\textwidth]{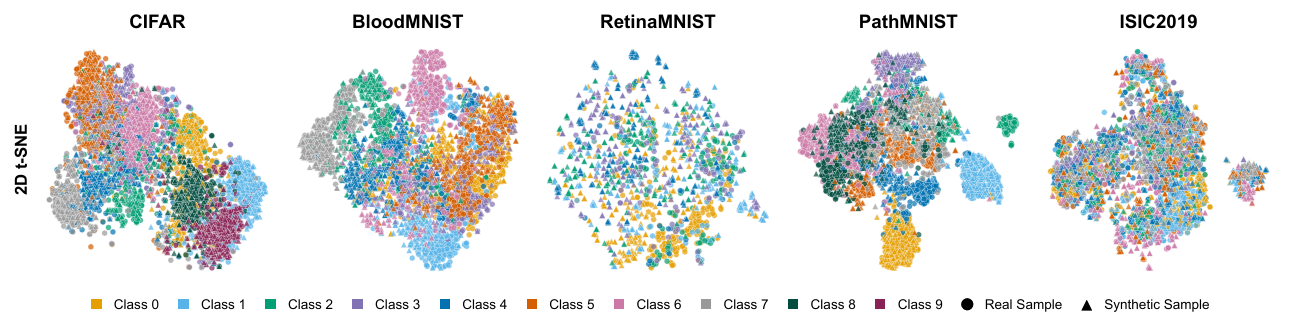}
    \caption{\textbf{Global Manifold Alignment in 2D t-SNE Subspace.} 
Feature space projections comparing true localized training ($\mathcal{D}_{\mathrm{union}}$, $\circ$) and synthesized global ($\hat{\mathcal{D}}_{\mathrm{syn}}$, $\triangle$) distributions across all five benchmarks.}
    \label{fig:tsne}
\end{figure}

\begin{figure}[htpb!]
    \centering
    \includegraphics[width=\textwidth]{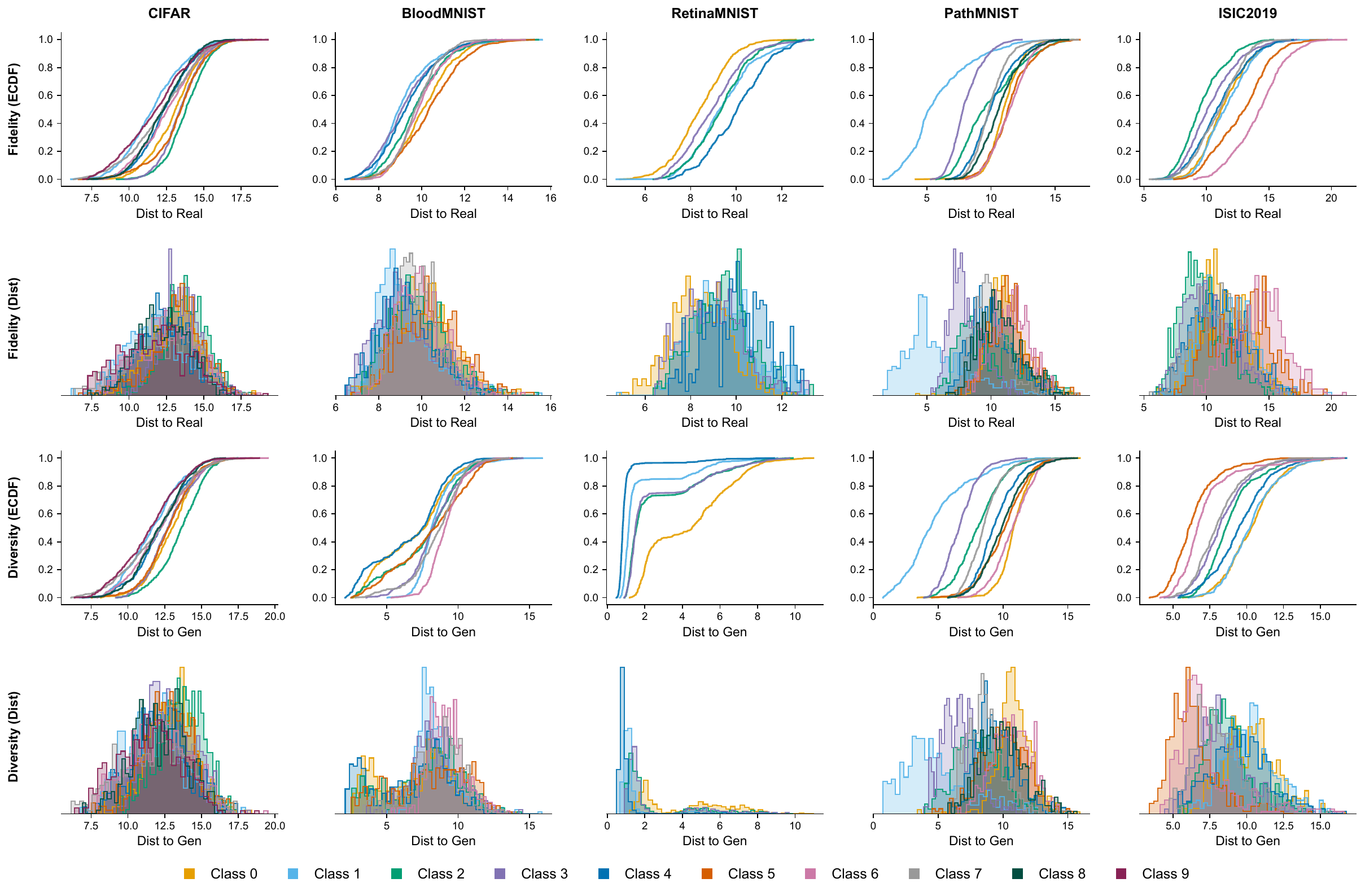}
    \caption{\textbf{Quantitative Manifold Alignment via ECDF and Density Histograms.} Evaluation of inter-class \textit{Fidelity} (top rows: L2 distance from synthetic samples to their nearest real neighbors) and intra-class \textit{Diversity} (bottom rows: L2 distance strictly among synthetic samples) across five benchmarks.}
    \label{fig:generative_metrics}
\end{figure}

\begin{itemize}
    \item \textbf{Fidelity (Distance to Real):} Quantified by the minimum Euclidean (L2) distance in the feature space from each synthetic sample to its absolute nearest real neighbor ($\hat{x} \to x$). 
    \item \textbf{Diversity (Distance to Generated):} Measured by the L2 distance strictly among the generated samples themselves ($\hat{x}_i \to \hat{x}_j$), useful to identify center collapse.
\end{itemize}

Steep Fidelity ECDF curves confirm that synthetic distributions tightly bind to the real data manifold with minimal Out-of-Distribution (OOD) deviation. Simultaneously, diversity ECDFs and density histograms shows how the generative factories maintain broad continuous support without mode collapse. These curves also reflect intrinsic domain variances: tightly grouped distributions in \textsc{RetinaMNIST} \cite{Yang2023MedMNIST} mirror the low structural variance of retinal crops, while broader ECDFs in \textsc{ISIC2019} \cite{Terrail2022FLamby} and \textsc{CIFAR} \cite{Krizhevsky2009CIFAR} capture their high visual heterogeneity. Ultimately, these metrics mathematically prove \textsc{FederatedFactory} captures the true underlying data support despite extreme single-class label skew.

\section{Additional FederatedFactory Performance Analysis}
\label{app:performance}

To illustrate the catastrophic failure of standard parameter aggregation \cite{McMahan2017FedAvg} and our framework's subsequent recovery, \cref{fig:federatedfactory_improvement_barplot} compares performance magnitudes across all five benchmarks. Under the pathological \textit{Single-Class Silo} regime, severe gradient interference causes baselines (FedAvg \cite{McMahan2017FedAvg}, FedDyn \cite{Acar2021FedDyn}, FedProx \cite{Li2020FedProx}, SCAFFOLD \cite{Karimireddy2020SCAFFOLD}) to collapse into near-random or majority-class predictions \cite{Yu2020GradientSurgery}. The barplot visualizes the $\Delta$ improvement (arrows) from the best iterative baseline to \textsc{FederatedFactory}. Ultimately, our zero-dependency approach completely bridges this gap, restoring predictive metrics to the theoretical data-pooled upper bound.

\begin{figure}[t]
    \centering
    \includegraphics[width=\textwidth]{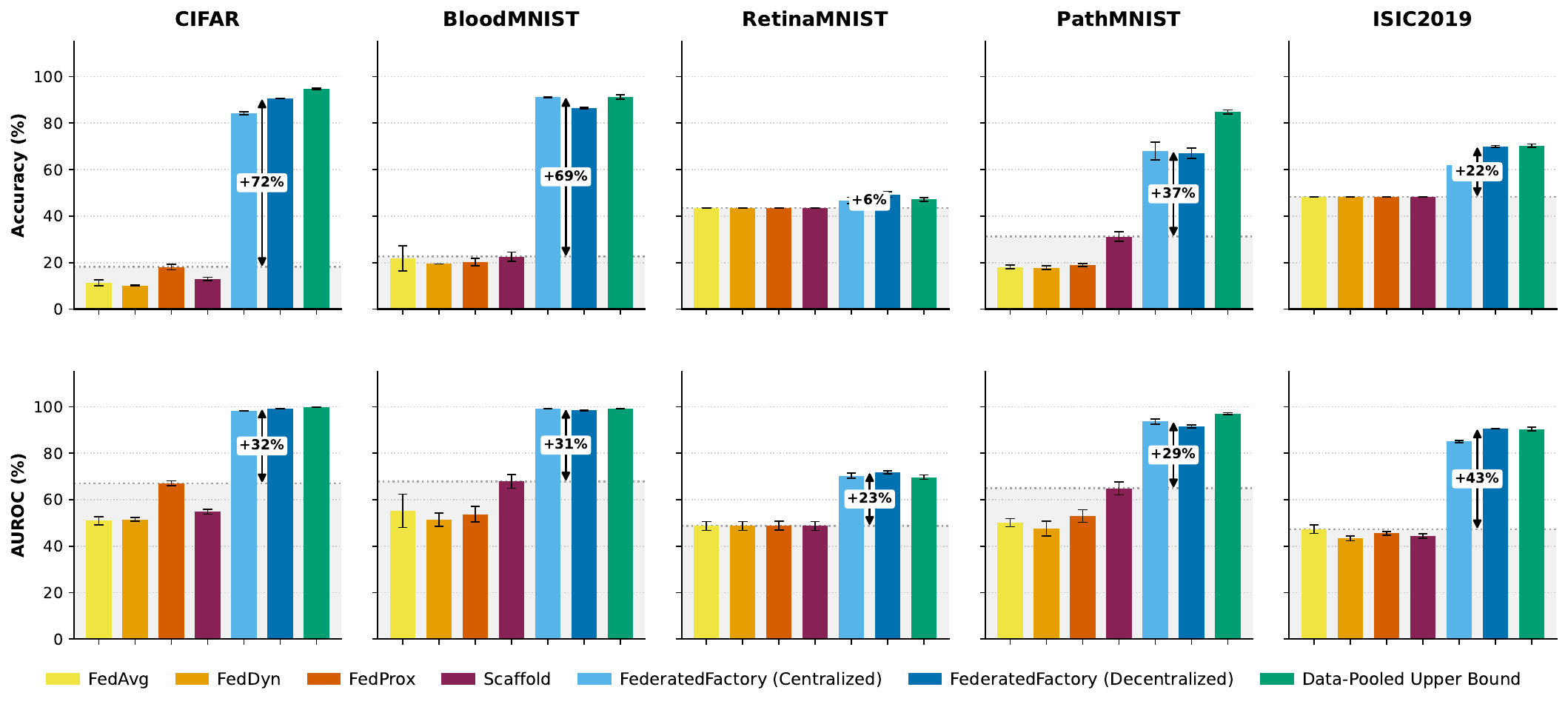}
    \caption{\textbf{Absolute Performance Recovery under Single-Class Silos.} Grouped bar chart detailing the accuracy and AUROC across the evaluated datasets. The gray shaded region represents the collapsed performance ceiling of standard iterative federated learning methods. \textsc{FederatedFactory} (in both \legendbox{colorFedFactC!20}{\textbf{Centralized}} and \legendbox{colorFedFactD!20}{\textbf{Decentralized}} modes) successfully escapes this collapsed regime, yielding improvements (e.g., matching the \legendbox{colorUpperBound!20}{\textbf{Centralized Upper Bound}}).}
    \label{fig:federatedfactory_improvement_barplot}
\end{figure}

\section{Additional FederatedFactory Cost Analysis}
\label{app:cost}

\textsc{FederatedFactory} transitions from a bandwidth-bound optimization regime to a compute-bound one in order to achieve higher model trustworthiness. We visualize this ``Compute-for-Bandwidth Swap'' via log-log scatter plot in \cref{fig:costs_plot}.

Standard FL \cite{McMahan2017FedAvg, Acar2021FedDyn, Li2020FedProx, Karimireddy2020SCAFFOLD} minimizes local compute ($\sim 10^{16}$--$10^{17}$ FLOPs) but requires massive iterative communication ($>10^5$ MB). Conversely, \textsc{FederatedFactory} employs a One-Shot FL (OSFL) protocol, trading a ``generative tax'' in local compute ($\sim 10^{18}$--$10^{20}$ FLOPs) for a $>99.4\%$ reduction in network payload ($\sim 10^3$ MB). This compute-for-bandwidth trade-off is ideal for data-siloed clinical consortiums possessing abundant local compute.

\begin{figure}[htpb!]
    \centering
    \includegraphics[width=\textwidth]{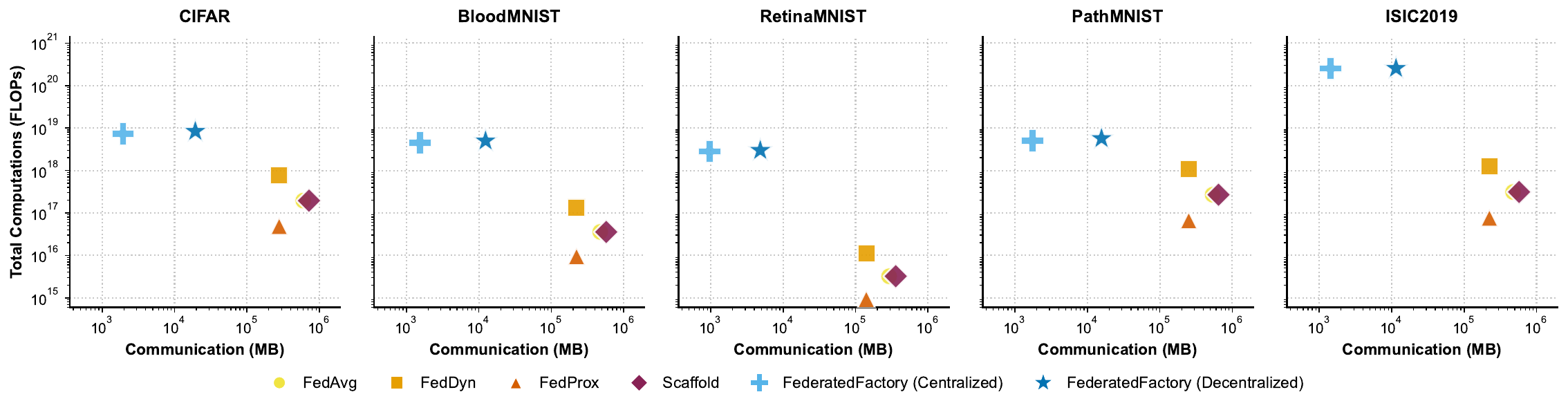}
\caption{\textbf{The Compute-for-Bandwidth Swap.} Log-log scatter comparing computational overhead versus network exposure. Iterative baselines (\legendbox{colorFedAvg!20}{\textbf{FedAvg $\circ$}}, \legendbox{colorFedDyn!20}{\textbf{FedDyn $\square$}}, \legendbox{colorFedProx!20}{\textbf{FedProx $\triangle$}}, \legendbox{colorScaffold!20}{\textbf{Scaffold $\diamond$}}) cluster bottom-right (low compute, massive bandwidth). Conversely, \textsc{FederatedFactory} (\legendbox{colorFedFactC!20}{\textbf{Centralized $+$}}, \legendbox{colorFedFactD!20}{\textbf{Decentralized $\star$}}) shifts to the top-left, accepting higher local FLOPs to achieve higher trustworthiness.}
\label{fig:costs_plot}
\end{figure}


\section{Image Generation Metrics}
\label{app:image_generation_metrics}
To quantify synthesized data quality, fidelity, and diversity, we evaluate generated ($\hat{\mathcal{D}}_{\mathrm{syn}}$) against target distributions ($\mathcal{D}_{\mathrm{union}}$) using Fréchet Inception Distance (FID) \cite{Heusel2017Gans} and Kernel Inception Distance (KID) \cite{Binkowski2018Demystifying}. These project images into the Inception-V3 activation space \cite{Szegedy2016Rethinking} to measure statistical divergence between real and synthetic manifolds, improving upon naive pixel-space distances. As \cref{tab:image_generation_metrics_table} and \cref{fig:image_generation_metrics} illustrates, per-class FID and KID distributions confirm \textsc{FederatedFactory} captures coherent semantic structures. On natural domains (\textsc{CIFAR-10} \cite{Krizhevsky2009CIFAR}), per-class FID averages $\sim$28.8, with KID reliably $<0.02$ ($\times 100$). Similarly, high-resolution dermatoscopic lesions (\textsc{ISIC2019} \cite{Terrail2022FLamby}) maintain strong manifold alignment with median FIDs $\sim$45 and tightly grouped KIDs. Benchmarks with complex, sparse cellular morphologies (\textsc{PathMNIST}, \textsc{BloodMNIST} \cite{Yang2023MedMNIST}) exhibit slightly higher variances.

\definecolor{colorCifar}{HTML}{004D40}
\definecolor{colorBlood}{HTML}{CC79A7}
\definecolor{colorRetina}{HTML}{E69F00}
\definecolor{colorPath}{HTML}{56B4E9}
\definecolor{colorIsic}{HTML}{8172B3}

\begin{table}[htpb]
    \centering
    \setlength{\tabcolsep}{4pt}
    \caption{\textbf{Per-Class Generative Alignment Metrics.} Quantitative evaluation of the synthesized images via FID and KID (scaled by $10^2$). Lower values indicate better manifold alignment with the real target distribution. Values are rounded to two decimal places for readability. For CIFAR, classes 0--9 correspond to airplane, automobile, bird, cat, deer, dog, frog, horse, ship, and truck, respectively.}
    \label{tab:class_wise_generation_metrics}
    \resizebox{\textwidth}{!}{%
    \begin{tabular}{ll cccccccccc}
        \toprule
        \multirow{2}{*}{\textbf{Dataset}} & \multirow{2}{*}{\textbf{Metric}} & \multicolumn{10}{c}{\textbf{Class}} \\
        \cmidrule(lr){3-12}
        & & \textbf{0} & \textbf{1} & \textbf{2} & \textbf{3} & \textbf{4} & \textbf{5} & \textbf{6} & \textbf{7} & \textbf{8} & \textbf{9} \\
        \midrule
        
        \multirow{2}{*}{\legendbox{colorCifar!20}{\textbf{CIFAR}}} 
        & \textbf{FID} ($\downarrow$) & 29.07 & 22.50 & 27.28 & 46.67 & 22.10 & 35.40 & 41.02 & 15.01 & 19.32 & 25.75 \\
        & \textbf{KID} ($\downarrow$) & 1.71 & 1.38 & 1.59 & 3.48 & 1.50 & 2.16 & 3.09 & 0.58 & 1.00 & 1.85 \\
        \midrule
        
        \multirow{2}{*}{\legendbox{colorBlood!20}{\textbf{BloodMNIST}}} 
        & \textbf{FID} ($\downarrow$) & 84.01 & 61.54 & 45.08 & 50.37 & 60.58 & 66.42 & 50.78 & 27.19 & -- & -- \\
        & \textbf{KID} ($\downarrow$) & 10.76 & 8.83 & 4.63 & 6.33 & 7.54 & 7.69 & 6.76 & 2.97 & -- & -- \\
        \midrule
        
        \multirow{2}{*}{\legendbox{colorRetina!20}{\textbf{RetinaMNIST}}} 
        & \textbf{FID} ($\downarrow$) & 38.49 & 58.40 & 38.77 & 44.09 & 58.88 & -- & -- & -- & -- & -- \\
        & \textbf{KID} ($\downarrow$) & 3.52 & 3.03 & 2.85 & 2.90 & 3.91 & -- & -- & -- & -- & -- \\
        \midrule
        
        \multirow{2}{*}{\legendbox{colorPath!20}{\textbf{PathMNIST}}} 
        & \textbf{FID} ($\downarrow$) & 38.20 & 33.85 & 77.25 & 103.85 & 57.24 & 64.21 & 64.68 & 79.69 & 70.86 & -- \\
        & \textbf{KID} ($\downarrow$) & 3.83 & 2.92 & 8.95 & 14.33 & 6.75 & 7.01 & 7.63 & 10.43 & 8.94 & -- \\
        \midrule
        
        \multirow{2}{*}{\legendbox{colorIsic!20}{\textbf{ISIC2019}}} 
        & \textbf{FID} ($\downarrow$) & 38.80 & 41.18 & 30.01 & 40.91 & 45.54 & 76.67 & 82.39 & 55.55 & -- & -- \\
        & \textbf{KID} ($\downarrow$) & 1.59 & 2.18 & 1.46 & 1.23 & 1.90 & 1.85 & 1.76 & 1.92 & -- & -- \\
        
        \bottomrule
    \end{tabular}%
    \label{tab:image_generation_metrics_table}
    }
\end{table}

\begin{figure}[htpb]
    \centering
    \includegraphics[width=\textwidth]{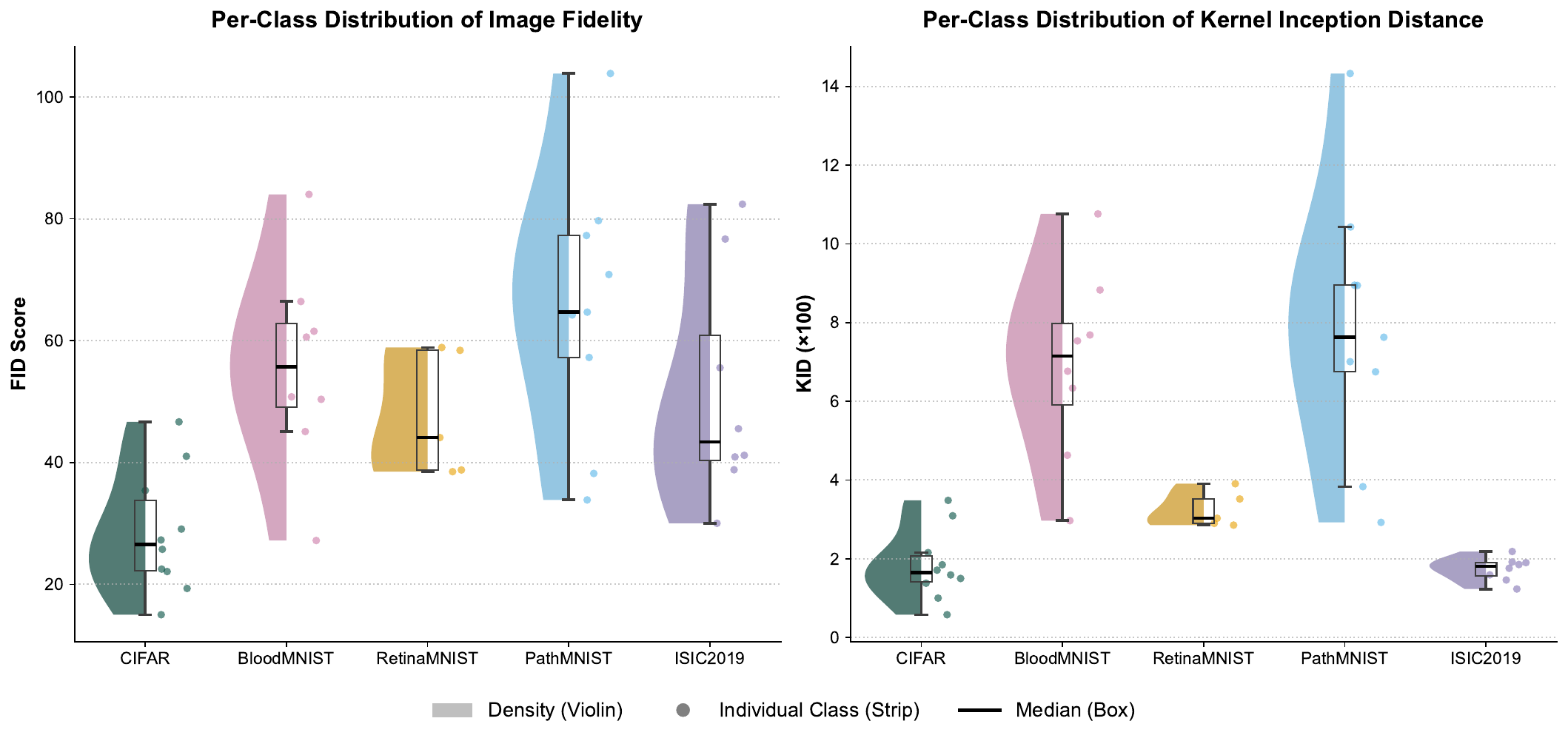}
    \caption{\textbf{Quantitative Manifold Alignment via FID and KID.} Per-class distributions of FID (left) and KID (right). The raincloud plots visualize the density, raw data points, and median scores.}
    \label{fig:image_generation_metrics}
\end{figure}

These quantitative metrics represent a strict lower bound on the framework's synthesis capability. To maintain computational tractability on distributed nodes, we deployed a heavily downscaled EDM2~\cite{Karras2024EDM2} architecture (128 embedding dimension, 32 sampling steps) and faced severe long-tail class imbalances (e.g., \textsc{ISIC2019}). Despite these sub-optimal conditions, \textsc{FederatedFactory} synthesizes counterfactuals matching centralized upper-bound classification accuracy (\cref{sec:experiments}). Because global discriminative performance is mathematically bounded by local generative error (\cref{theorem:convergence}), achieving robust downstream classification with non-ideal generators guarantees significant upside. As local compute scales, enabling larger foundation diffusion models and extended training, \textit{ex nihilo} synthesis fidelity will proportionally improve, yielding an even stronger global decision boundary.

\section{Computational Environments}
\label{app:computational_environments}

The core system is equipped with 2$\times$ NVIDIA H100 and 4$\times$ NVIDIA L40 GPUs.

\begin{itemize}
    \item \textbf{NVIDIA H100 (80GB VRAM):} The two H100 GPUs were exclusively dedicated to training the localized generative priors, specifically the EDM2 diffusion models \cite{Karras2022EDM, Karras2024EDM2}. 
    
    \item \textbf{NVIDIA L40 (48GB VRAM):} The four L40 GPUs were parallelized to handle server-side image synthesis and downstream evaluations. Specifically, they were utilized to synthesize the globally class-balanced datasets \textit{ex nihilo}, and for training the global discriminative classifiers, i.e. the adaptive ResNet-50 \cite{He2016ResNet} backbone.
\end{itemize}

\section{Datasets and Sources}
\label{app:datasets_and_sources}

We gratefully acknowledge the creators, curators, and institutions behind the public datasets used to evaluate \textsc{FederatedFactory}. The empirical validation of our method relies entirely on the public availability of these benchmarks (CIFAR-10 \cite{Krizhevsky2009CIFAR}, MedMNIST \cite{Yang2023MedMNIST}, ISIC2019 \cite{Terrail2022FLamby})

\begin{figure}[ht]
    \centering
    \includegraphics[width=\textwidth]{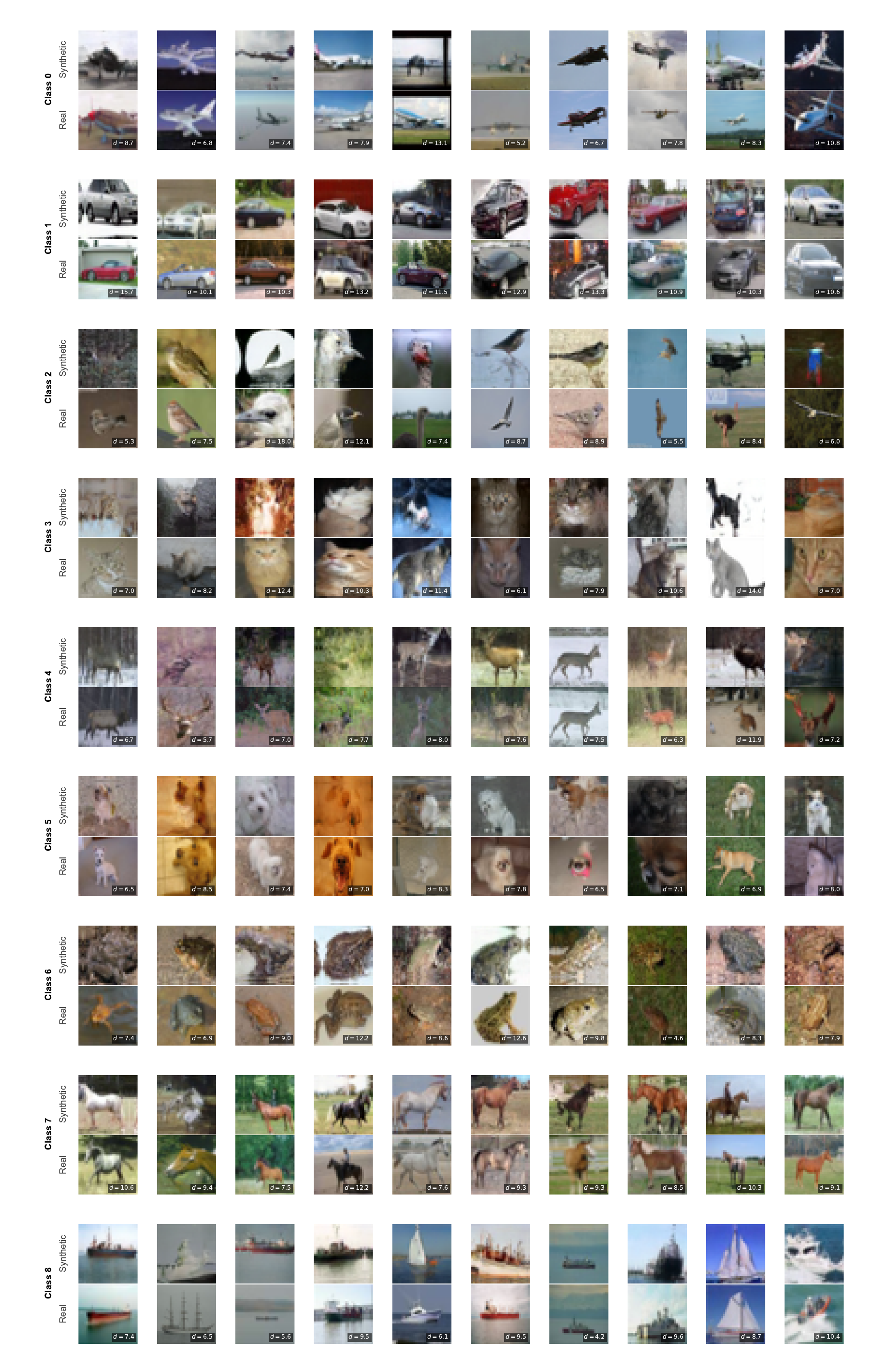}
    \caption{\textbf{Nearest Neighbor Analysis on CIFAR-10.} Synthetic samples (top rows) paired with their closest real counterparts (bottom rows). Class 9 has been omitted to accommodate layout constraints.}
    \label{fig:comparison_cifar}
\end{figure}

\begin{figure}[ht]
    \centering
    \includegraphics[width=\textwidth]{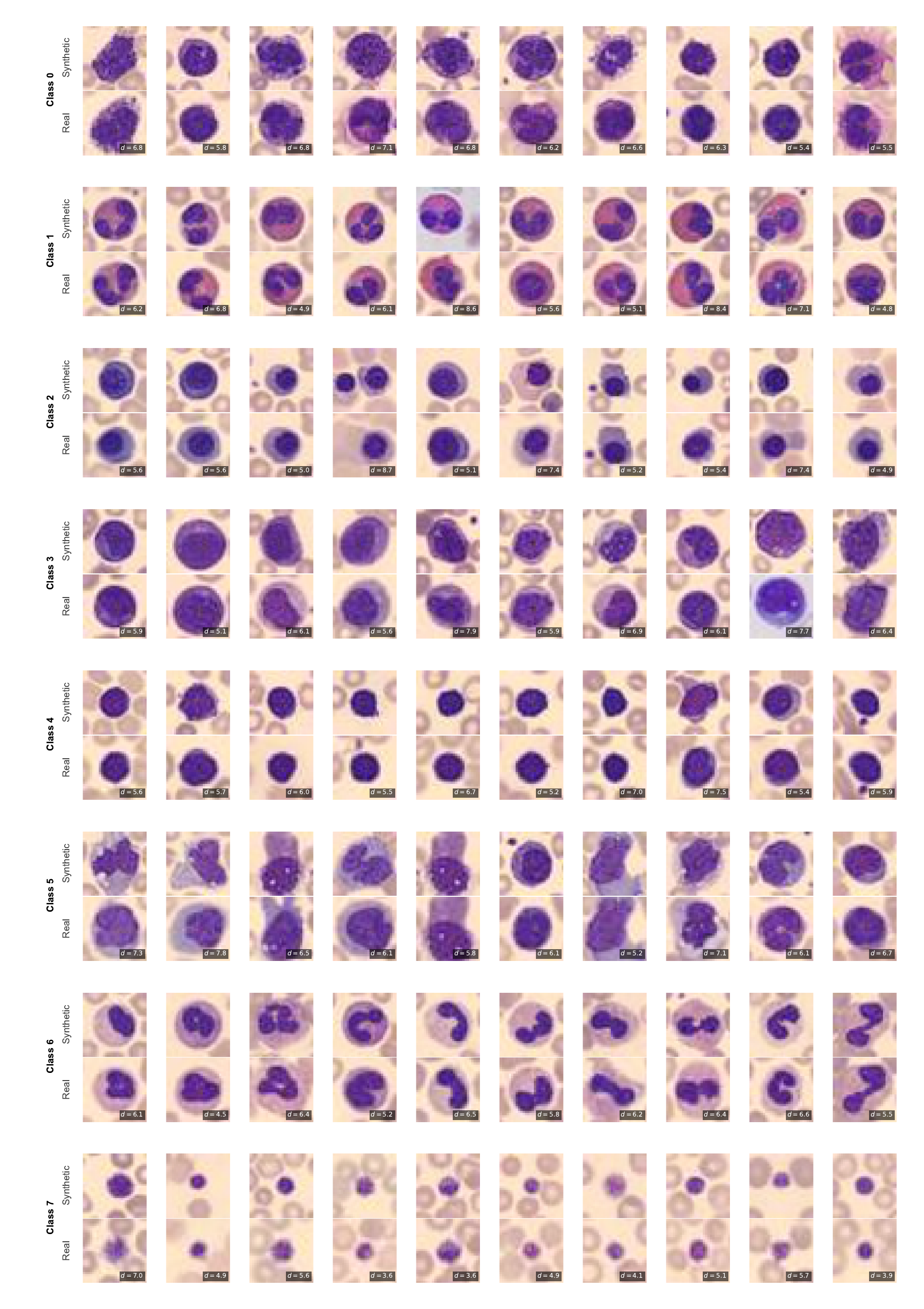}
    \caption{\textbf{Nearest Neighbor Analysis on BloodMNIST.} The localized Factories successfully capture the distinct morphological features, shapes, and textures of different blood cell types.}
    \label{fig:comparison_bloodmnist}
\end{figure}

\begin{figure}[ht]
    \centering
    \includegraphics[width=\textwidth]{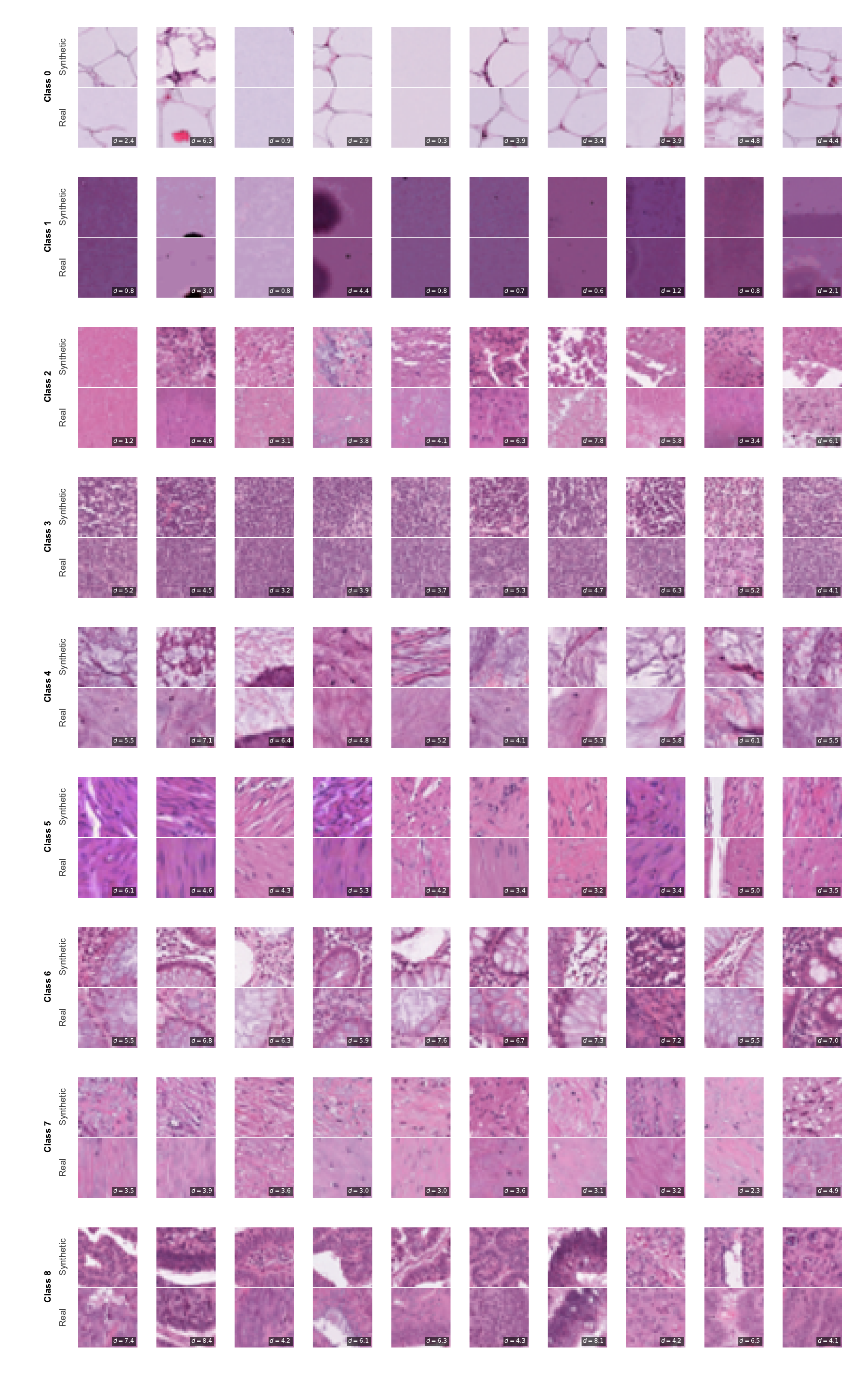}
    \caption{\textbf{Nearest Neighbor Analysis on PathMNIST.} Synthetic histological patches alongside their closest real training samples.}
    \label{fig:comparison_pathmnist}
\end{figure}

\begin{figure}[ht]
    \centering
    \includegraphics[width=\textwidth]{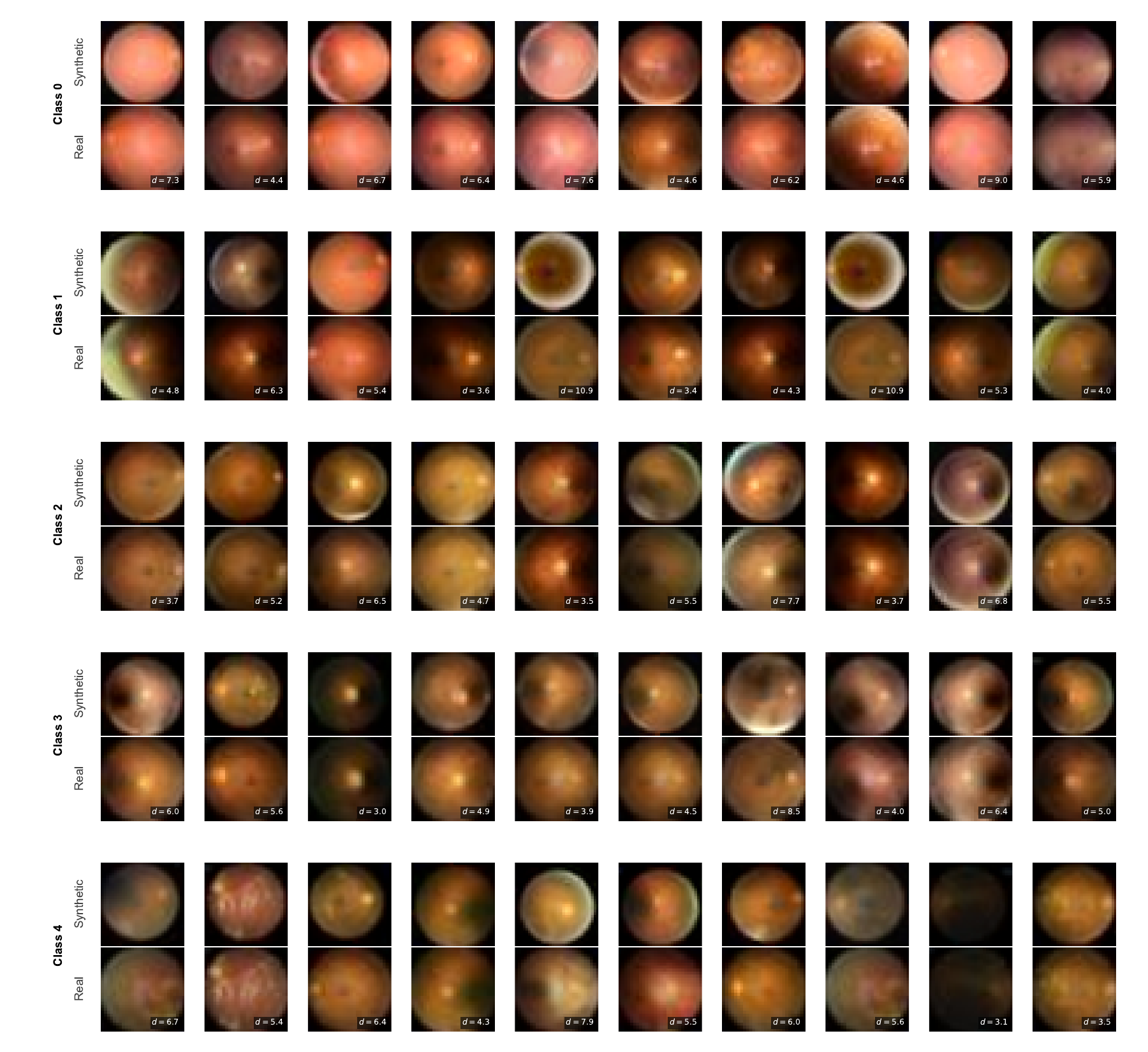}
    \caption{\textbf{Nearest Neighbor Analysis on RetinaMNIST.} Despite the intrinsically low structural variance of retinal fundus crops, the generative models maintain sufficient diversity and do not trivially memorize the limited training support.}
    \label{fig:comparison_retinamnist}
\end{figure}

\begin{figure}[ht]
    \centering
    \includegraphics[width=\textwidth]{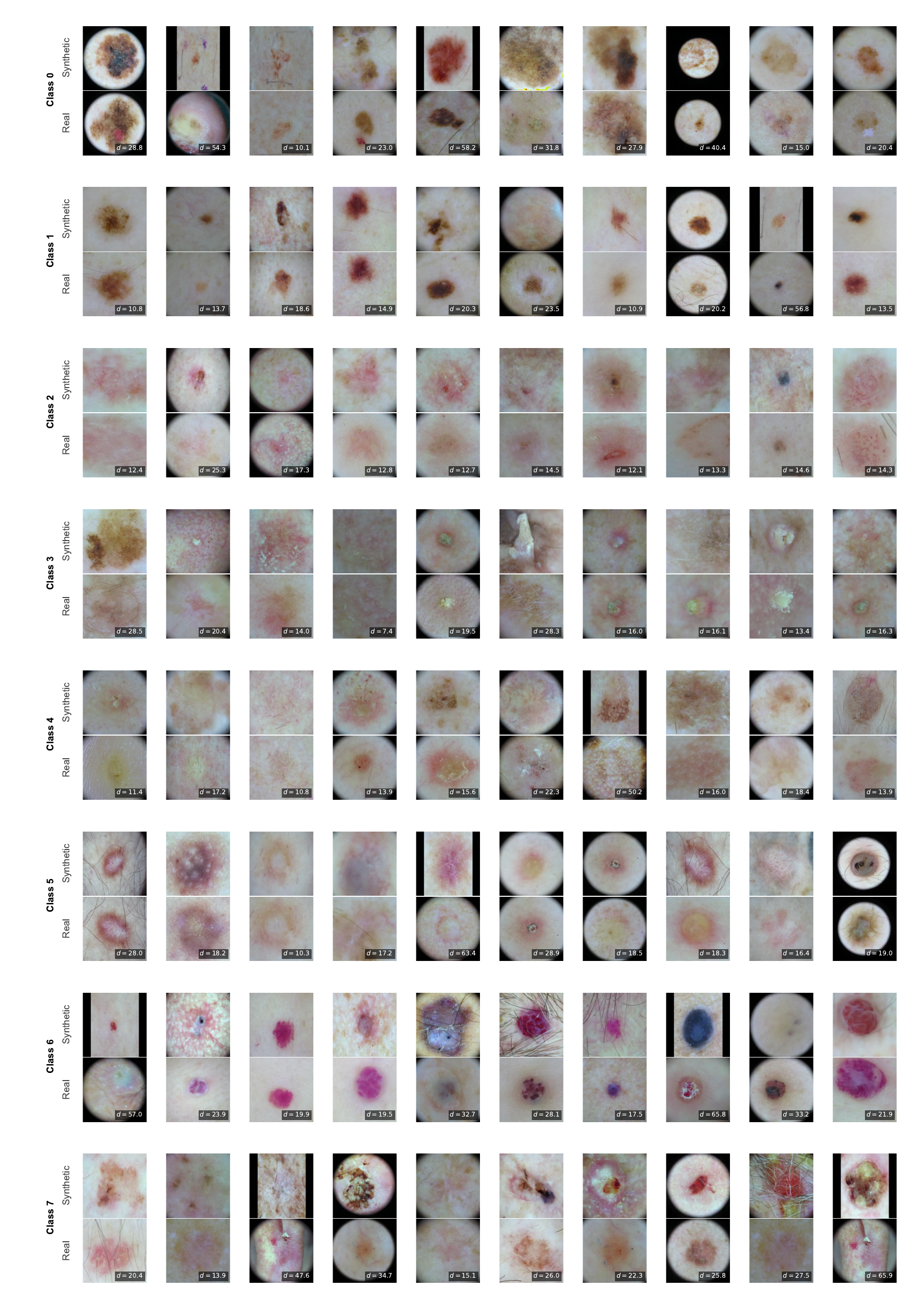}
    \caption{\textbf{Nearest Neighbor Analysis on ISIC2019.} High-resolution dermatoscopic lesions are synthesized with high fidelity.}
    \label{fig:comparison_isic}
\end{figure}

\end{document}